\documentclass[10pt,twocolumn,letterpaper]{article}

\usepackage{cvpr}              
\definecolor{cvprblue}{rgb}{0.21,0.49,0.74}
\usepackage[pagebackref,breaklinks,colorlinks,allcolors=cvprblue]{hyperref}

\usepackage{hyperref}
\usepackage{url}
\hypersetup{
    colorlinks=true,
    linkcolor=red,
    urlcolor=cyan,
}
\usepackage{amsfonts}       
\usepackage{amsthm}         
\usepackage{nicefrac}       
\usepackage{microtype}      
\usepackage{graphicx}
\usepackage{amsmath}
\usepackage{algorithm}
\usepackage{algpseudocode}
\usepackage{pifont}
\usepackage{fontawesome5}
\usepackage{tabularx}
\usepackage{multirow}
\usepackage{wrapfig}
\usepackage{booktabs}
\usepackage{colortbl}
\usepackage{xspace}
\usepackage{amssymb}
\usepackage{wrapfig}
\usepackage{enumitem}
\usepackage[accsupp]{axessibility}

\newcommand{\cmark}{\ding{51}} 
\newcommand{\xmark}{\ding{55}} 

\usepackage{caption}

\newtheorem{theorem}{Theorem}

\newtheorem{lemma}[theorem]{Lemma}
\newtheorem{assumption}{Assumption}
\theoremstyle{definition}
\newtheorem{definition}{Definition}
\newtheorem*{remark}{Remark}

\usepackage{tikz}

\definecolor{midgray}{gray}{0.60}

\newcommand{\recttikzl}[4][1.8em]{%
  \tikz[baseline=(r.base)] 
    \node (r)
      [draw=#2, fill=#3, line width=#4,
       inner sep=0pt,
       minimum width=#1, minimum height=1.0em,
       text height=1.5ex, text depth=.25ex] {};%
}

\definecolor{Green}{HTML}{3d9f3c}
\definecolor{customgreen}{HTML}{008000}

\definecolor{RED}{HTML}{e9212c}
\definecolor{BLUE}{HTML}{1181b2}
\newcolumntype{Y}{>{\centering\arraybackslash}X}
\newcolumntype{Z}{>{\raggedleft\arraybackslash}X}

\def\paperID{{*****}} 
\def\confName{CVPR}
\def\confYear{2026}

\title{
Prime Once, then Reprogram Locally: \\
An Efficient Alternative to Black-Box Service Model Adaptation
}

\author{Yunbei Zhang$^{1}$ \quad Chengyi Cai$^{2}$ \quad Feng Liu$^{2}$ \quad Jihun Hamm$^{1}$\\
$^{1}$Tulane University \quad $^{2}$University of Melbourne
}

\begin{document}
\maketitle

\begin{abstract}
Adapting closed-box service models (i.e., APIs) for target tasks typically relies on reprogramming via Zeroth-Order Optimization (ZOO). However, this standard strategy is known for extensive, costly API calls and often suffers from slow, unstable optimization. Furthermore, we observe that this paradigm faces new challenges with modern APIs (e.g., GPT-4o). These models can be less sensitive to the input perturbations ZOO relies on, thereby hindering performance gains. To address these limitations, we propose an Alternative efficient Reprogramming approach for Service models (AReS). Instead of direct, continuous closed-box optimization, AReS initiates a single-pass interaction with the service API to prime an amenable local pre-trained encoder. This priming stage trains only a lightweight layer on top of the local encoder, making it highly receptive to the subsequent glass-box (white-box) reprogramming stage performed directly on the local model. Consequently, all subsequent adaptation and inference rely solely on this local proxy, eliminating all further API costs. Experiments demonstrate AReS's effectiveness where prior ZOO-based methods struggle: on GPT-4o, AReS achieves a +27.8\% gain over the zero-shot baseline, a task where ZOO-based methods provide little to no improvement. Broadly, across ten diverse datasets, AReS outperforms state-of-the-art methods (+2.5\% for VLMs, +15.6\% for standard VMs) while reducing API calls by over 99.99\%. AReS thus provides a robust and practical solution for adapting modern closed-box models. Code: \url{https://github.com/yunbeizhang/AReS}.
\end{abstract}

\section{Introduction}
\label{sec:intro}
In recent years, Model-as-a-Service (MaaS) has emerged as a dominant paradigm for deploying and accessing state-of-the-art (SOTA) machine learning models \citep{gpt4,brown2020language,CLIP,gemini}. Companies and research institutions often release their large pre-trained models (which we refer to as \emph{service models} in this work) as inference APIs, allowing users to query these powerful models without direct access to model parameters or architecture details. This closed-box nature creates a significant challenge for transfer learning \citep{pmlr-v235-sun24p, sun2022bbt, wang2025doctor}. Unlike some approaches that may assume access to intermediate features or token embeddings \citep{ouali2023black, wang2024craft, zhang2025dpcore} (see Table~\ref{tab:bb_comparison} for a detailed comparison), this work focuses on the most restrictive closed-box setting, where only raw input-prediction access to the service model is available. In this challenging scenario, adapting models to downstream tasks, which typically requires gradient information, becomes particularly difficult. Model Reprogramming (MR) \citep{cai2024sample,chen2024model, jia2022visual, bar_icml20, park2025zip}, known as Visual Reprogramming (VR) or Visual Prompting (VP) in vision domains, has been developed as one solution, enabling adaptation of MaaS models through input-level modifications optimized using zeroth-order optimization (ZOO) \citep{ liu2018zeroth,spall1992multivariate, spall1997one,tu2019autozoom} based solely on model outputs.

While closed-box visual reprogramming approaches such as BAR~\citep{bar_icml20} and BlackVIP~\citep{balckvip_CVPR23} have demonstrated promising results under this strict input-prediction only setting, they entail practical challenges. The optimization process, reliant on approximate gradients, can be slow and unstable, and the computational and practical costs could be prohibitive \citep{9186148, zhang2024revisiting}. Such methods require numerous API calls not only during the lengthy training phase but also for every inference instance, leading to significant financial burdens and requiring uninterrupted connectivity to the service model \citep{balckvip_CVPR23, bar_icml20, wang2024craft}. 
More critically, our experiments indicate that this ZOO-based paradigm faces new challenges with modern, robust APIs. 
We find that powerful models like GPT-4o can be less sensitive to the noisy input perturbations central to ZOO, yielding little to no performance gains (Fig.~\ref{fig:gpt4o} and Table~\ref{tab:mllm_api}), while on models like LLaVA~\cite{llava}, these perturbations may affect delicate vision-language alignment.
Furthermore, the substantial resources expended for these SOTA methods do not always translate into commensurate performance enhancements (Fig.~\ref{fig:api_comparison} and \ref{fig:time_comparison}), raising concerns about their efficiency-to-effectiveness ratio \citep{balckvip_CVPR23, sun_ijcai22}. 
These substantial hurdles in efficiency, cost, and now, effectiveness on modern models, form the core motivation for our work. It compels us to ask: \emph{Are there alternative approaches that can offer advantages beyond conventional closed-box model reprogramming, particularly in overcoming these practical and economic limitations?}



\begin{figure*}[t]
    \centering
    \subfloat[\textbf{GPT-4o Cost (USD)}]{
      \begin{minipage}[b]{0.32\linewidth} 
        \centering  
        \includegraphics[width=\linewidth]{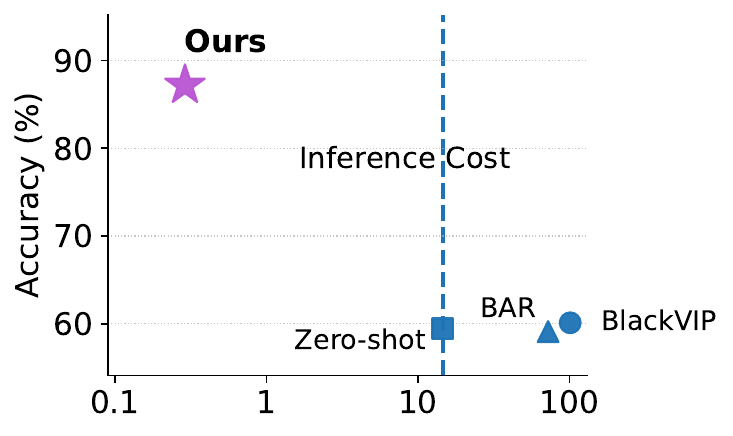} 
        \vspace{-4mm}
        \label{fig:gpt4o}
      \end{minipage}
    }
    \subfloat[\textbf{\#API Calls}]{
      \begin{minipage}[b]{0.32\linewidth} 
        \centering  
        \includegraphics[width=\linewidth]{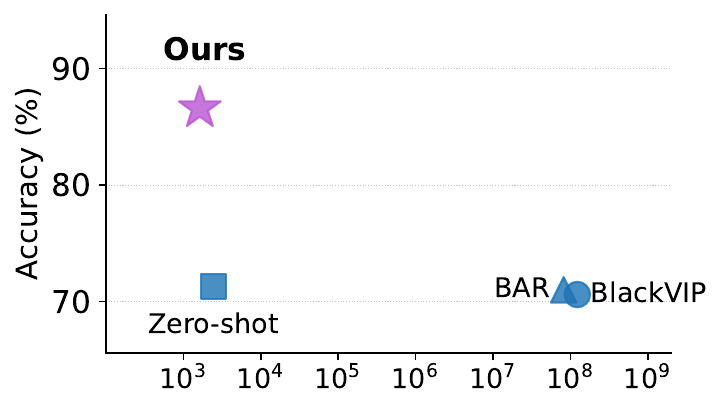}   
        \vspace{-4mm}
        \label{fig:api_comparison}
      \end{minipage}
    }
    \subfloat[\textbf{Training Time (hour)}]{
      \begin{minipage}[b]{0.32\linewidth} 
        \centering  
        \includegraphics[width=\linewidth]{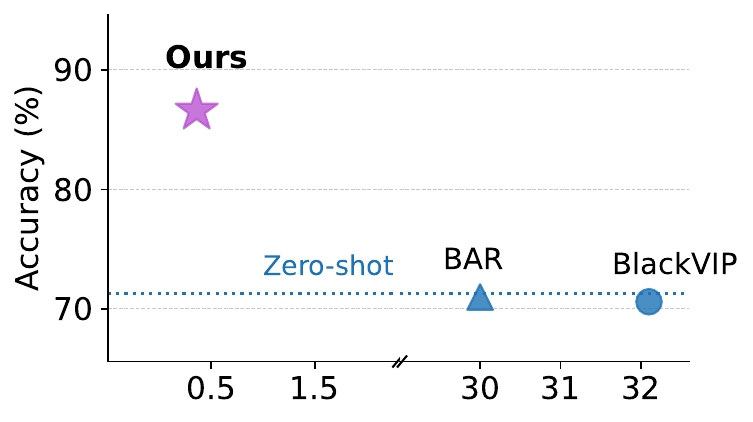}   
        \vspace{-4mm}
        \label{fig:time_comparison}
      \end{minipage}
    }
    \vspace{-1.5mm}
    \caption{
    \textbf{(a)} On the real-world GPT-4o API, ZOO-based methods show limited effectiveness, providing little to no improvement over zero-shot performance while incurring high total (training and inference) costs. \textbf{(b, c)} On CLIP ViT-B/16, these methods require $\sim$10$^8$ API calls and over 32 hours of training on Flowers102, yet still underperform ours, AReS, which uses only $\sim$10$^3$ API calls and 0.4 hours.}
    \vspace{-3mm}
\end{figure*}

We propose an \textbf{\underline{A}}lternative efficient \textbf{\underline{Re}}programming approach for \textbf{\underline{S}}ervice models (\textbf{AReS}) that directly addresses these challenges (see Fig.~\ref{fig:algo_motivation}). Instead of direct, continuous closed-box optimization, AReS initiates a single-pass interaction with the service API to prepare an amenable local pre-trained encoder (e.g., a publicly available one, making the same initial assumption as methods like BlackVIP~\cite{balckvip_CVPR23} but utilizing it for a fundamentally different purpose). This priming stage trains only a lightweight layer on top of the local encoder. Crucially, this step does not require a shared label space; it successfully prepares the local model for reprogramming even when the service API's outputs are in a disjoint label space (e.g., ImageNet) from the target task (e.g., Flowers).
This approach, supported by a theoretical analysis connecting priming faithfulness to downstream performance, enables effective glass-box visual reprogramming to occur entirely and efficiently on the local model. Consequently, all subsequent adaptation and inference rely solely on this local model, eliminating further API costs.
For example, on the Flowers102 dataset (as detailed in Fig.~\ref{fig:api_comparison} and \ref{fig:time_comparison}), our approach achieves 86.6\% accuracy, outperforming BlackVIP (70.6\%), while dramatically reducing API calls from $\sim10^8$ down to $\sim10^3$ ($>$\textbf{99.99}\% reduction) and slashing computation time from over 30 hours to less than half an hour ($>$\textbf{98.88}\% reduction).  
Crucially, on GPT-4o, AReS improves by $+27.8\%$, while BlackVIP gains only $+0.7\%$. This is because the service model's robustness to input perturbations hinders ZOO-based methods, a limitation AReS bypasses by avoiding direct reprogramming.

We evaluate AReS across a broad range of service models. Our primary evaluation spans standard Vision Models (VMs) (e.g., ViT-B/16~\citep{ViT}, ResNet101~\citep{resnet}) and Vision-Language Models (VLMs) (e.g., CLIP~\citep{CLIP}) over ten diverse datasets. We also conduct evaluations on modern models, including the Multimodal Large Language Models (MLLMs) LLaVA~\cite{llava}, the proprietary MLLM API GPT-4o~\cite{gpt4}, and the commercial VM API Clarifai. Through these comprehensive experiments, we demonstrate that our approach consistently outperforms SOTA ZOO-based methods while dramatically reducing computational requirements and associated costs. These findings have important implications for the broader field of transfer learning in increasingly API-centric machine learning ecosystems, suggesting a more effective path forward that balances performance, efficiency, and practicality.
Our main contributions are summarized as follows:
\begin{enumerate}
    \item We identify a key challenge for ZOO-based reprogramming on modern, robust APIs and present an effective alternative that shifts adaptation to a local model after a minimal, single-pass interaction with the service API.

    \item We introduce \textbf{AReS}, a simple, theoretically-supported framework that combines a priming step with efficient glass-box reprogramming, uniquely enabling \emph{cost-free} inference without further API dependency.

    \item Through extensive experiments, AReS achieves superior performance, including significant gains on real-world APIs (e.g., GPT-4o and Clarifai) where prior methods fail, alongside average accuracy improvements of $+2.5\%$ for VLM and $+15.6\%$ for VM over BlackVIP, while reducing API calls by $>99.99\%$.
\end{enumerate}

\begin{figure*}
    \centering
    \includegraphics[width=1.0 \linewidth]{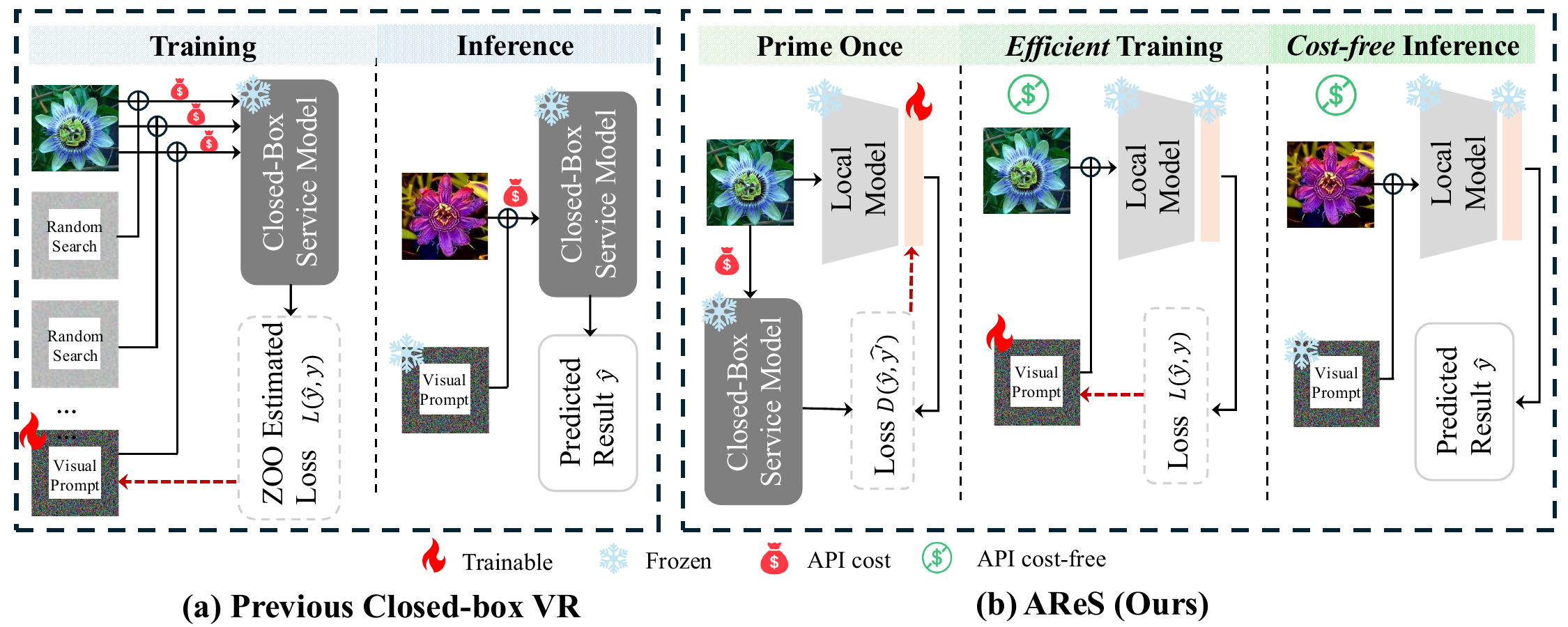}
    \vspace{-6mm}
    \caption{
    \textbf{(a)} Previous closed-box methods use Zeroth-Order Optimization (ZOO), which estimates gradients by querying the API multiple times for the same training image with random perturbations. This results in continuous API dependency: numerous calls during training and one API call per image during inference. \textbf{(b)} Our AReS approach performs a single-pass priming, requiring only one API call per training image to prepare a local model. This enables efficient, gradient-aware VR locally, eliminating all subsequent API costs during inference.}
    \label{fig:algo_motivation}
    \vspace{-4mm}
\end{figure*}

\section{Related Work}
\label{sec:related_work}
\textbf{Closed-box Tuning for Model-as-a-Service (MaaS).}
Closed-box tuning for MaaS has been explored across various modalities, with approaches for Large Language Models (LLMs) \citep{pmlr-v235-sun24p, sun2022bbt} and methods for VLMs that modify inputs at the token level while requiring access to text embeddings, placing them outside the strictest closed-box definition~\citep{guo-etal-2023-black, ouali2023black, wang2024craft, sun_ijcai22}, 
as detailed in Table~\ref{tab:bb_comparison}. 
In the vision domain, approaches like BlackVIP \citep{balckvip_CVPR23} and BAR \citep{bar_icml20} employ ZOO to reprogram models by modifying visual inputs. While applicable to both VLMs and VMs, these ZOO-based approaches are known for optimization instability and high API call costs during training and inference~\citep{9186148, balckvip_CVPR23, bar_icml20, zhang2024revisiting, park2025zip}. More critically, their reliance on input perturbations may face new challenges in the era of modern APIs, e.g., powerful service models like GPT-4o can be robust enough to simply ignore these perturbations, rendering the costly ZOO adaptation process ineffective (Fig.~\ref{fig:gpt4o}). 
These limitations highlight the need for an alternative strategy that avoids direct closed-box optimization and instead enables efficient and effective adaptation for downstream tasks without repeated service-model queries.

\noindent\textbf{Model Reprogramming (MR).}
MR~\citep{bahng2022exploring,chen2024model, jia2022visual,  zhang2022fairness}, termed Visual Reprogramming (VR) or Visual Prompting (VP) in vision domains, encompasses two main strategies. The first, token-based prompting, is tailored for transformer architectures like Vision Transformers (ViT)~\citep{bahng2022exploring, jia2022visual, NEURIPS2024_0a0eba34, zhou2022conditional, zhou2022learning, xiao2026promptbased}. These methods inject learnable tokens and require access to the model's embedding space or intermediate layers, rendering them unsuitable for strict closed-box settings or non-transformers~\citep{guo-etal-2023-black,ouali2023black,wang2024craft,sun_ijcai22, zhang2025otvp}. The second category, input-level prompting~\citep{cai2024bayesian, chen2023understanding, balckvip_CVPR23,bar_icml20,tsao2024autovp, zhang2026adapting}, directly modifies input images, offering model-agnostic flexibility with only input-output access. Both approaches can be optimized using gradient-based methods (glass-box) or ZOO techniques (closed-box)~\citep{balckvip_CVPR23, bar_icml20, wang2024craft}. 
However, ZOO-based methods can be associated with efficiency and stability trade-offs. Their reliance on high query volumes and continuous API connectivity may be less suited for practical scenarios requiring rapid adaptation or offline use, such as in remote field operations. These practical considerations particularly motivate our work in exploring an alternative strategy.

\begin{table*}[!t]
\caption{
Comparison of tuning methods for Model-as-a-Service (MaaS). The Access column indicates the strictness, distinguishing translucent-box (grey, requiring intermediate Features or Embedding access) from closed-box (black, Input/Predictions only) settings. We compare methods by: MR (Model Reprogramming-based); required Input Access and Output Access; support for VMs (Standard Vision Models) and VLMs (Vision-Language Models); Prompt type (L: Language, V: Visual); and Encoder Flexibility. Our method, AReS, uniquely offers Cost-free Inference for both VMs and VLMs while operating under the strictest closed-box constraints.
}
\vspace{-2mm}
\label{tab:bb_comparison}
\centering
\scriptsize
\resizebox{\textwidth}{!}{%
\begin{tabularx}{\textwidth}{lXcYYcccYY} 
\toprule
\textbf{Access} & \textbf{Method} & \textbf{MR} & \textbf{Input Access} & \textbf{Output Access} & \textbf{VM} & \textbf{VLM} & \textbf{Prompt} & \textbf{Encoder Flexibility} & \textbf{Cost-free Inference}\\
\midrule
\multirow{5}{*}{\recttikzl{midgray}{midgray}{1.0pt}} & LFA \citep{ouali2023black}          & \xmark & Input                   & \cellcolor{gray!30}Features      & \xmark & \cmark & --      & \cmark      & \xmark \\
&CBBT \citep{guo-etal-2023-black}    & \xmark & \cellcolor{gray!30}Embedding & \cellcolor{gray!30}Features      & \xmark & \cmark & L       & \cmark      & \xmark \\

&LaViP \citep{kunananthaseelan2024lavip}    & \cmark & Input & \cellcolor{gray!30}Features      & \cmark & \cmark & V       & \cmark      & \xmark \\

&BPT-VLM~\citep{sun_ijcai22}     & \xmark & \cellcolor{gray!30}Embedding & Predictions                  & \xmark & \cmark & L\&V    & ViT Only    & \xmark \\
&CraFT \citep{wang2024craft}         & \xmark & \cellcolor{gray!30}Embedding & Predictions                  & \xmark & \cmark & L       & \cmark      & \xmark \\
\midrule 
\multirow{4}{*}{\recttikzl{black}{black}{1.0pt}} & BAR \citep{bar_icml20}              & \cellcolor{green!20}\cmark & \cellcolor{green!20}Input      & \cellcolor{green!20}Predictions      & \cmark & \cmark & V       & \cmark      & \xmark \\
&BlackVIP \citep{balckvip_CVPR23}    & \cellcolor{green!20}\cmark & \cellcolor{green!20}Input      & \cellcolor{green!20}Predictions      & \cmark & \cmark & V       & \cmark      & \xmark \\

&LLM-Opt \citep{liu2024language}    & \xmark & \cellcolor{green!20}Input      & \cellcolor{green!20}Predictions      & \xmark & \cmark & L       & \cmark      & \xmark \\

&\textbf{AReS (Ours)}                        & \cellcolor{green!20}\cmark & \cellcolor{green!20}Input      & \cellcolor{green!20}Predictions      & \cmark & \cmark & V       & \cmark      & \cellcolor{green!20}\cmark \\
\bottomrule
\end{tabularx}
}
\vspace{-3mm}
\end{table*}

\section{Problem Setting}
\textbf{Closed-box Model Reprogramming.}
MR adapts pre-trained models to solve new tasks without modifying their parameters. Given a pre-trained model $\mathcal{F}_S: \mathcal{X}^S \mapsto \mathbb{R}^{K^S}$ where $K^S$ is the number of source classes, the key concept is to introduce a learnable prompt $\mathbf{P}$ that transforms inputs from the downstream task domain $\mathcal{X}^T$ to the source domain $\mathcal{X}^S$ through an input transformation function $g_{\mathrm{in}}(\cdot|\mathbf{P})$. Additionally, a label mapping function $g_{\mathrm{out}}:\mathbb{R}^{K^S}\mapsto\mathbb{R}^{K^T}$ is employed to bridge the source (with $K^S$ classes) and target (with $K^T$ classes) label spaces. If the source and downstream tasks have the same label space, $g_{\mathrm{out}}$ becomes an identity mapping (e.g., in VLMs). Otherwise, for VMs, a gradient-free label mapping is established following existing works \citep{cai2024bayesian, chen2023understanding, bar_icml20} (Detailed in Appendix~\ref{app:output_mapping}). Importantly, only the input transformation $g_{\mathrm{in}}$ contains trainable parameters updated through gradient descent, while the label mapping $g_{\mathrm{out}}$ is computed and updated during training but does not introduce additional trainable parameters. 

In MaaS settings, powerful pre-trained models are accessible only through APIs without access to their parameters or internal representations. We denote such a closed-box Service API model as $\mathcal{F}_S: \mathcal{X}^S \mapsto \mathbb{R}^{K^S}$, which accepts inputs from domain $\mathcal{X}^S$ and produces output predictions. Adhering to the strictest closed-box assumption established in prior work \citep{bar_icml20}, we consider the scenario where the API returns only prediction probabilities rather than logits or embeddings, representing the most constrained access scenario (see Table~\ref{tab:bb_comparison}). When adapting MaaS models to downstream tasks, the central challenge is to leverage these powerful models in a computationally and economically efficient manner. This requires developing methods that can optimize task-specific visual prompts using only the limited input-output responses from the API while minimizing computational costs and API usage during both training and inference.

\noindent\textbf{Existing BMR Approaches.} 
Existing ZOO-based BMR methods like BAR~\citep{bar_icml20} employs random gradient-free (RGF) optimization \citep{liu2018zeroth, tu2019autozoom} to estimate the gradient of the loss function $\ell$ with respect to prompt parameters $\mathbf{P}$: $\nabla \ell(\mathbf{P}) \approx \frac{d}{q}\sum_{i=1}^{q} \frac{\ell(\mathbf{P} + \mu u_i) - \ell(\mathbf{P})}{\mu}u_i\,,$
where $d$ is the parameter dimension, $q$ is the number of random directions, $\mu$ is a small constant, and $u_i$ are random unit vectors. BlackVIP~\citep{balckvip_CVPR23} introduces the Simultaneous Perturbation Stochastic Approximation with Gradient Correction (SPSA-GC) \citep{spall1992multivariate, spall1997one} to improve gradient estimation: $\nabla \ell(\mathbf{P}) \approx \frac{\ell(\mathbf{P} + c\Delta) - \ell(\mathbf{P} - c\Delta)}{2c}\Delta^{-1}\,,$
where $c$ is a small constant and $\Delta$ is a random perturbation vector. BlackVIP also incorporates a Coordinator network for input-dependent prompt generation using a local pre-trained encoder (e.g., ViT-B/16). While these methods have shown promising results, they require multiple API calls per update and during inference, leading to increased computational and financial costs. Additionally, their zeroth-order gradient approximation can result in less stable optimization compared to exact gradients (i.e., first-order optimization) \citep{9186148,balckvip_CVPR23,  zhang2024revisiting}. The challenge remains to develop more efficient approaches that can better leverage available local encoders while minimizing API usage.

\begin{table*}[t]
\caption{Accuracy and Efficiency comparison on ten diverse visual recognition datasets using CLIP ViT-B/16 (Service) and ViT-B/16 (Local) in a 16-shot setting. AReS achieves superior average accuracy with a drastic reduction in resources (99.99\% fewer API calls and 98\% less training time). For this and all subsequent tables, \textcolor{gray!50}{Gray} denotes glass-box Service VR results. \#API (M): total API calls in millions (training and inference); Wall-clock Time (h): total training hours run on an RTX 6000 Ada GPU.}
\vspace{-2mm}
\centering
\scriptsize
\resizebox{\textwidth}{!}{
\begin{tabular}{l|cccccccccc|c|r|r}
\toprule
Method & Flowers & DTD & UCF & Food & GTSRB & EuroSAT & Pets & Cars & SUN & SVHN &  Avg. & \#API (M) & Time (h) \\
\midrule
VR (glass-box)  
& \textcolor{gray!50}{86.8} & \textcolor{gray!50}{62.0} & \textcolor{gray!50}{74.0} & \textcolor{gray!50}{81.6} & \textcolor{gray!50}{65.1} & \textcolor{gray!50}{90.9} & \textcolor{gray!50}{90.7} & \textcolor{gray!50}{66.2} & \textcolor{gray!50}{66.7} & \textcolor{gray!50}{60.2} & \textcolor{gray!50}{74.4} & \textcolor{gray!50}{16.82} & \textcolor{gray!50}{9.8} \\
\midrule
Zero-shot 
& \underline{71.3} & 43.9 & 66.9 & \textbf{85.9} & 21.0 & 47.9 & \textbf{89.1} & \underline{65.2} & 62.6 & 17.9 & 57.2 & 0.12 & 0 \\
BAR 
& 71.0 & \underline{46.8} & 64.2 & 84.4 & \underline{21.5} & \underline{77.3} & 88.4 & 63.0 & 62.4 & 34.6 & 61.4 & 612.84 & 185.6 \\
BlackVIP 
& 70.6 & 45.3 & \textbf{68.7} & \textbf{85.9} & 21.3 & 73.3 & \textbf{89.1} & \textbf{65.4} & \textbf{64.5} & \underline{44.4} & 62.9 & 754.20 & 197.5 \\
LLM-Opt &
67.6&	45.0&	59.9&	78.5&	21.2&	48.0&	87.7&	56.2&	60.3&	20.2&	54.5&5011.32&5.1 \\
\cellcolor{gray!30}\textbf{AReS}     
& \cellcolor{gray!30}\textbf{86.6} & \cellcolor{gray!30}\textbf{48.2} & \cellcolor{gray!30}\underline{67.1} & \cellcolor{gray!30}68.8 & \cellcolor{gray!30}\textbf{39.4} & \cellcolor{gray!30}\textbf{85.7} & \cellcolor{gray!30}88.9 & \cellcolor{gray!30}43.2 & \cellcolor{gray!30}\underline{62.8} & \cellcolor{gray!30}\textbf{63.2} & \cellcolor{gray!30}\underline{65.4} & \cellcolor{gray!30}\textbf{0.02} & \cellcolor{gray!30}\textbf{3.7} \\

\cellcolor{gray!30}\textbf{AReS-MS}     
& \cellcolor{gray!30}\textbf{86.6} & \cellcolor{gray!30}\textbf{48.2} & \cellcolor{gray!30}\underline{67.1} & \cellcolor{gray!30}\textbf{85.9} & \cellcolor{gray!30}\textbf{39.4} & \cellcolor{gray!30}\textbf{85.7} & \cellcolor{gray!30}88.9 & \cellcolor{gray!30}\underline{65.2} & \cellcolor{gray!30}\underline{62.8} & \cellcolor{gray!30}\textbf{63.2} & \cellcolor{gray!30}\textbf{69.3} & \cellcolor{gray!30}\textbf{0.06} & \cellcolor{gray!30}\textbf{3.7} \\
\bottomrule
\end{tabular}
}
\vspace{-3mm}
\label{tab:clip_res}
\end{table*}

\section{Method}
\textbf{Prime Once.}
In the first stage, we utilize a locally available encoder $\mathcal{F}_L$ (the same assumption as BlackVIP \citep{balckvip_CVPR23}) and prime it using the closed-box service model $\mathcal{F}_S$. For each image $x_i$ from the target domain dataset $\mathcal{D}^T = \{x_i\}_{i=1}^N$, we query the service model to obtain prediction probabilities ${p}_S(x_i) = \mathcal{F}_S(x_i)$. It is important to note that this is the \emph{only time} we interact with the API, and after this single round of querying, the API is \emph{no longer needed} for the rest of the process.
We use the local encoder $\mathcal{F}_L$ (outputting $d_{\mathrm{enc}}$-dim features) as a fixed feature extractor and add a trainable linear layer with parameters $\theta \in \mathbb{R}^{K^S \times (d_{\mathrm{enc}}+1)}$ on top of it. Note that we abuse the notation $\mathcal{F}_L$ to refer to both the local encoder alone and the local encoder combined with the learned linear layer, depending on the context. Only the linear layer parameters are updated during priming:
\begin{equation} 
\label{eq:distillation_loss}
\scalebox{0.95}{$
    \theta^* = \arg\min_\theta \sum_{i=1}^N \mathcal{L}_{\text{P}}(\mathcal{F}_L(x_i; \theta), {p}_S(x_i)) \,.
$}
\end{equation}
Our priming objective, inspired by the loss function from knowledge distillation \citep{hinton2015distilling, stanton2021does}, is defined as: $\mathcal{L}_{\text{P}}({p}_L, {p}_S) := -\sum_{j=1}^{K^S} {p}_{S,j} \log {p}_{L,j}\,,$
where ${p}_S/{p}_L$ are the output probabilities of the same image from the service/local models, respectively. This loss is equivalent to the Kullback-Leibler divergence $\text{KL}({p}_S||{p}_L)$.

For both VMs and VLMs, we freeze the local encoder and only train the added linear layer. The key difference between VM and VLM adaptation is in how the outputs are handled: for VMs, the priming occurs in the \emph{service model label probability space}, which presents a non-trivial challenge as this fixed vocabulary does not align with the downstream task's labels. For VLMs, the outputs are naturally aligned with the downstream task label space through the text encoder. This approach establishes a fair comparison with BlackVIP~\citep{balckvip_CVPR23}, as both methods only require a local encoder. 
Moreover, this opens up the potential to utilize extra \textbf{unlabeled} data from the downstream distribution, a significant advantage unavailable to other methods. (see Fig.~\ref{fig:ablation_extra_distil}).

\begin{remark}[Priming vs. Knowledge Distillation] \label{sec:remark_kd}
Our priming stage is mechanistically similar to knowledge distillation \citep{ba2014deep, chen2017learning, hinton2015distilling, hsu2021generalization, ji2020knowledge, phuong2019towards} but differs fundamentally in its objective and applicability. Whereas traditional distillation aims to create a final, high-performance student model---typically requiring a shared label space---our priming is solely a preparatory step to make the local model more amenable to subsequent reprogramming. This preparatory focus allows priming to operate even when the service and downstream label spaces are disjoint (e.g., using Flowers data in the ImageNet label space), providing crucial flexibility for closed-box model adaptation.
\end{remark}

\noindent\textbf{Efficient VR Training.}
After the priming, we proceed to the second stage of AReS: learning an effective visual prompt for the downstream task using gradient-based optimization. Unlike closed-box approaches that rely on approximate gradients, we can leverage exact gradient information since our local encoder with the learned linear layer is fully accessible.
We define a learnable visual prompt $\mathbf{P}$ and an input transformation function $g_{\mathrm{in}}$ similar to closed-box methods, along with a label mapping function $g_{\mathrm{out}}$ (and becomes an identity mapping for VLMs). We optimize the prompt using gradient descent on our primed local encoder by minimizing the cross-entropy loss:
\begin{equation} 
\label{eq:vr_training}
\scalebox{0.85}{$
    \mathbf{P}^* = \arg\min_{\mathbf{P}} \mathbb{E}_{(x,y) \sim \mathcal{D}^T} [\ell(g_{\mathrm{out}}(\mathcal{F}_L(g_{\mathrm{in}}(x, \mathbf{P}); \theta^*)), y)]\,.
$}
\end{equation}
This gradient-based optimization on local model enables more stable, efficient, and rapid learning of the visual prompt compared to ZOO-based methods, resulting in better and faster adaptation to the downstream task while maintaining the benefits of input-level modifications.

\noindent\textbf{Cost-free Inference.}
For inference on test examples from the downstream task, we apply the learned visual prompt to the input and process it through the local encoder $\hat{y} = g_{\mathrm{out}}(\mathcal{F}_L(g_{\mathrm{in}}(x, \mathbf{P}^*); \theta^*))$.
This approach eliminates API dependency during inference, enabling truly cost-free deployment. By removing the need for constant connectivity to service models, AReS becomes particularly valuable in practical scenarios like remote field operations or edge devices with intermittent network access.

\noindent\textbf{Model Selection Strategy.} For applications where maximizing performance is the top priority, we also propose a simple variant named \textbf{AReS-MS} (Model Selection). This approach leverages the initial priming stage as a low-cost diagnostic tool, as it yields both the zero-shot performance of the service API and the potential performance of the primed local model. AReS-MS then automatically selects the optimal inference path: if the local model's potential meets the required performance threshold, defined as the zero-shot performance minus a tolerance $\tau$, it proceeds with the cost-free local model; otherwise, it defaults to using the more powerful zero-shot API. This turns AReS into a more complete and robust framework that not only enables efficient adaptation but also guides practitioners to the optimal cost-performance trade-off for each individual task.

\noindent\textbf{Theoretical Insight.}
To formally understand AReS's efficacy, we provide a theoretical analysis (detailed in Appendix~\ref{app:theoretical_analysis}) that bounds the performance of our primed local model, $\mathcal{F}_L$, relative to the service model, $\mathcal{F}_S$. This analysis hinges on two key assumptions: \textit{Service Model Superiority} (Assumption~\ref{app:assum_service_model_superiority}), which reasonably posits that $\mathcal{F}_S$ is more powerful or better aligned with the downstream task data $\mathcal{D}^T$ than $\mathcal{F}_L$ initially, and \textit{$\epsilon$-Faithful Priming} (Assumption~\ref{app:assum_faithful_distill} and Definition~\ref{app:def_faithful_distill}). The latter states our one-time priming effectively aligns the output logit distributions of $\mathcal{F}_L$ and $\mathcal{F}_S$, with their expected $L_1$ difference bounded by a small $\epsilon$, indicating a successful knowledge transfer. While our theoretical analysis assumes \emph{logits} for tractability, our ablation studies (Fig.~\ref{fig:ablation_ditil_loss}) confirm that the performance of our practical, probability-only method is unaffected.

Leveraging these assumptions and the Lipschitz continuity of the cross-entropy loss (Lemma~\ref{app:lemma_lipschitz_cross_entropy_general}), Theorem~\ref{app:thm_bound_service_model_perf} establishes a crucial performance bound. It shows that the risk of the optimally reprogrammed local model, $\mathcal{R}_L(\mathcal{D}^T, \mathbf{P}^*)$, can closely approach that of an optimally reprogrammed service model, $\mathcal{R}_S(\mathcal{D}^T, \mathbf{Q}^*)$, differing by at most $\epsilon$:
\begin{equation} \label{eq:theo_bound}
\scalebox{0.9}{$
    \mathcal{R}_L(\mathcal{D}^T, \mathbf{P}^*) - \epsilon \le \mathcal{R}_S(\mathcal{D}^T, \mathbf{Q}^*) \le \mathcal{R}_L(\mathcal{D}^T, \mathbf{P}^*)\,.
    $}
\end{equation}
This result offers a key insight: traditional ZOO-based methods expend considerable resources attempting to directly optimize for a low $\mathcal{R}_S(\mathcal{D}^T, \mathbf{Q}^*)$ via unstable, query-heavy processes on $\mathcal{F}_S$. AReS, however, transforms this challenge. By first achieving a faithful priming (small $\epsilon$), the problem shifts to optimizing $\mathcal{R}_L(\mathcal{D}^T, \mathbf{P}^*)$ on the local, glass-box model $\mathcal{F}_L$. This local optimization using efficient First-Order Optimization (FOO) is more stable and drastically reduces API calls and computation, thus providing a more practical path to leveraging the service model's capabilities.

\begin{table*}[t]
\caption{Accuracy and Efficiency comparison using ViT-B/16 (Service) in Full-shot setting.}
\vspace{-2mm}
\centering
\scriptsize
\resizebox{\textwidth}{!}{
\begin{tabular}{l|cccccccccc|c|r|r}
\toprule
Method & Flowers & DTD & UCF & Food & GTSRB & EuroSAT & Pets & Cars & SUN & SVHN &  Avg. & \#API (M) & Time (h) \\
\midrule
VR (glass-box)  
& \textcolor{gray!50}{69.6} & \textcolor{gray!50}{52.1} & \textcolor{gray!50}{49.7} & \textcolor{gray!50}{38.7} & \textcolor{gray!50}{60.6} & \textcolor{gray!50}{96.7} & \textcolor{gray!50}{77.1} & \textcolor{gray!50}{6.2} & \textcolor{gray!50}{34.0} & \textcolor{gray!50}{83.3} & \textcolor{gray!50}{56.8} & \textcolor{gray!50}{43.4} & \textcolor{gray!50}{39.6} \\
\midrule
BAR  
& 14.8 & 24.3 & 29.5 & 12.2 & 14.1 & 43.1 & 24.9 & 1.0 & 10.6 & 29.6 & 20.4 & 1,724.2 & 668.2 \\
BlackVIP (RN50)
& 15.2 & 42.3 & 34.0 & 14.5 & 13.9 & 67.3 & 63.6 & 3.7 & 22.6 & 30.8 & 30.8 & 2,586.2 & 686.9\\
\cellcolor{gray!30}\textbf{AReS (RN50)} 
& \cellcolor{gray!30}\textbf{43.5} & \cellcolor{gray!30}\textbf{43.4} & \cellcolor{gray!30}\textbf{39.4} & \cellcolor{gray!30}\textbf{25.9} & \cellcolor{gray!30}\textbf{48.7} & \cellcolor{gray!30}\textbf{84.3} & \cellcolor{gray!30}\textbf{73.9} & \cellcolor{gray!30}\textbf{5.8} & \cellcolor{gray!30}\textbf{19.9} & \cellcolor{gray!30}\textbf{73.7} & \cellcolor{gray!30}\textbf{45.9} 
& \cellcolor{gray!30}\textbf{0.2} & \cellcolor{gray!30}\textbf{8.7} \\
BlackVIP (ViT-B/32)
& 27.4 & 44.1 & 36.6 & 18.5 & 38.4 & 59.6 & 65.4 & 3.9 & 23.7 & 30.3 & 34.8 & 2,586.2 & 757.6 \\
\cellcolor{gray!30}\textbf{AReS (ViT-B/32)}     
& \cellcolor{gray!30}\textbf{53.6} & \cellcolor{gray!30}\textbf{46.9} & \cellcolor{gray!30}\textbf{41.4} & \cellcolor{gray!30}\textbf{27.3} & \cellcolor{gray!30}\textbf{58.8} & \cellcolor{gray!30}\textbf{95.6} & \cellcolor{gray!30}\textbf{65.0} & \cellcolor{gray!30}\textbf{4.7} & \cellcolor{gray!30}\textbf{26.9} & \cellcolor{gray!30}\textbf{83.5} & \cellcolor{gray!30}\textbf{50.4} & \cellcolor{gray!30}\textbf{0.2} & \cellcolor{gray!30}\textbf{22.4} \\
\bottomrule
\end{tabular}
}
\label{tab:vitb16_res}
\vspace{-3mm}
\end{table*}

\section{Experiments}
\label{sec:exp}
\textbf{Datasets and Models.} 
We evaluate AReS on ten diverse visual recognition datasets widely used in previous works~\citep{balckvip_CVPR23, cai2024bayesian}, spanning various domains: fine-grained categorization (Flowers102~\citep{flowers102}, StanfordCars~\citep{stanfordcars}), texture recognition (DTD~\citep{DTD}), action recognition (UCF101~\citep{UCF101}), food classification (Food101~\citep{Food101}), traffic sign recognition (GTSRB~\citep{GTSRB}), satellite imagery (EuroSAT~\citep{EuroSAT}), animal recognition (OxfordPets~\citep{Oxfordpets}), scene classification (SUN397~\citep{sun397}), and digit recognition (SVHN~\citep{SVHN}). We compare our approach with SOTA closed-box VR methods: BAR \citep{bar_icml20} and BlackVIP \citep{balckvip_CVPR23}. For the VLM experiments, we additionally compare against LLM-Opt~\citep{liu2024language}, a recent baseline applicable only to VLMs. For VLMs, we use CLIP ViT-B/16~\citep{CLIP} as the service model and ViT-B/16 as the local encoder in a 16-shot learning setting, following the same configuration as BlackVIP~\citep{balckvip_CVPR23} due to VLMs' strong generalization capabilities. We also include zero-shot CLIP performance as an important baseline to measure adaptation gains. For VMs, we employ pre-trained ViT-B/16~\citep{ViT} and ResNet101 (RN101)~\citep{resnet} as service models with ViT-B/32 and ResNet50 (RN50) as local encoders, adopting the Full-shot setting following BAR~\citep{bar_icml20} since VMs are not as generalizable as VLMs. For all settings using VMs as a service, we employ Bayesian-guided Label Mapping (BLM)~\citep{cai2024bayesian} as $g_{\mathrm{out}}$.  In addition, we further evaluate our methods on the MLLM LLaVA~\citep{llava}, the real-world API GPT-4o~\citep{hurst2024gpt}, and the commercial VM API service Clarifai. Additional details can be found in Appendix~\ref{app:exp_setting}.

\noindent\textbf{Implementation Details and Evaluation Metrics.} 
For the priming stage in AReS, we use the Adam optimizer \citep{adam} with a learning rate of 0.001. For local encoder VR, we follow the implementation in \citep{cai2024bayesian}, using a padding-based visual prompt approach and the Adam optimizer with a learning rate of 0.01. For baselines, we use the officially released implementations with their recommended hyperparameter settings. All experiments are conducted over \emph{three} rounds on a single NVIDIA RTX 6000 Ada GPU. We evaluate each method using three key metrics to assess practical deployment value: (1) Average accuracy across three runs on downstream tasks (\textbf{\%}), (2) Total training and inference API calls required in millions (\textbf{\#API (M)}), and (3) Total wall-clock training time in hours (\textbf{Time (h)}) across all datasets. For experiments on real-world APIs, we also report the total financial cost in US dollars (\$) for both training and inference.

\subsection{Experimental Results}
\textbf{Results on Vision-Language Model (VLMs).}
\label{subsec:vlm_res}
Following BlackVIP \citep{balckvip_CVPR23}, we evaluate our approach using a CLIP ViT-B/16 service model and a ViT-B/16 local encoder, with results presented in Table~\ref{tab:clip_res}. Our method, \textbf{AReS}, achieves a 65.4\% average accuracy across ten diverse datasets, outperforming both BAR (61.4\%) and BlackVIP (62.9\%). We also compare against LLM-Opt~\citep{liu2024language}, a recent method that uses GPT-4 to optimize text prompts; despite its extremely high API costs for both CLIP (5011.3M API calls) and the LLM optimizer (additional $\sim$\$110 USD), it underperforms the zero-shot baseline and is only applicable to VLMs. This superior performance by AReS is achieved with a drastic reduction in computational requirements: only 0.02M API calls for the initial priming, compared to 612.84M for BAR and 754.20M for BlackVIP (a reduction of $>99.99\%$). Furthermore, inferring all datasets' test sets (e.g., zero-shot) requires 0.12M API calls; our method uses only approximately 1/6 of these calls for priming and then requires no API calls during inference, while still outperforming zero-shot by $+8.2\%$ on average. Training time is also significantly reduced to only 3.7 hours, compared to 185.6 and 197.5 hours for BAR and BlackVIP, respectively (a reduction of $>98\%$).
These efficiency gains are particularly notable on challenging datasets such as SVHN (63.2\% vs. 44.4\% for BlackVIP and 34.6\% for BAR) and EuroSAT (85.7\% vs. 73.3\% for BlackVIP and 77.3\% for BAR). Our performance on Flowers102 is also strong at 86.6\%, substantially outperforming BlackVIP (70.6\%) and BAR (71.0\%), and closely approaching the glass-box VR performance (86.8\%) without direct use of the service model.

On certain complex datasets, notably Food101 and Cars, we observe that all evaluated reprogramming methods, including our own (with a failure case analysis provided in Appendix~\ref{app:failure_analysis}), struggle to outperform the strong zero-shot CLIP baseline. This suggests a potential inherent limitation of input-level visual reprogramming for these specific domains. This hypothesis is strongly supported by the results in Table~\ref{tab:clip_res}, where a fully glass-box VR approach is shown to be ineffective, even causing a 4.3\% performance \emph{decrease} on Food101 compared to the zero-shot baseline. This insight motivates our \textbf{AReS-MS} variant, which functions as an intelligent model selection strategy. By using the priming stage to identify these challenging datasets and defaulting to the zero-shot API, AReS-MS boosts the average accuracy to a new state-of-the-art 69.3\% (Table~\ref{tab:clip_res}), effectively navigating the cost-performance trade-off by applying reprogramming only where it is beneficial.

\begin{table*}[t]
\caption{Performance and Cost comparison on the open-source MLLM LLaVA, the proprietary API GPT-4o, and the commercial VM API Clarifai, evaluated on the EuroSAT 16-shot benchmark. A zero-shot baseline is not applicable for Clarifai as it is a closed-vocabulary model. Costs are shown in total API calls (\#) or US dollars (\$). AReS demonstrates significant gains where ZOO-based methods struggle.
}
\label{tab:mllm_api}
\vspace{-2mm}
\centering
\scriptsize
\resizebox{\textwidth}{!}{
\renewcommand{\arraystretch}{0.9}
\begin{tabular}{l|c|rrr|c|rrr|c|rrr}
\toprule
\multirow{2}{*}{Method} 
& \multicolumn{4}{c|}{\textbf{LLaVA}} 
& \multicolumn{4}{c|}{\textbf{GPT-4o}} 
& \multicolumn{4}{c}{\textbf{Clarifai}} \\
\cmidrule(lr){2-5}\cmidrule(lr){6-9}\cmidrule(lr){10-13}
& Acc. $\uparrow$ & Train & Infer. & Total (\#) $\downarrow$ 
& Acc. $\uparrow$ & Train & Infer. & Total (\$) $\downarrow$ 
& Acc. $\uparrow$ & Train & Infer. & Total (\$) $\downarrow$ \\
\midrule
glass-box
& \textcolor{gray!50}{91.3} & \textcolor{gray!50}{--} & \textcolor{gray!50}{--} & \textcolor{gray!50}{--} 
& \textcolor{gray!50}{--} & \textcolor{gray!50}{--} & \textcolor{gray!50}{--} & \textcolor{gray!50}{--} 
& \textcolor{gray!50}{--} & \textcolor{gray!50}{--} & \textcolor{gray!50}{--} & \textcolor{gray!50}{--} \\
\midrule
Zero-shot 
& 40.1 & -- & 8100 & 8100 
& 59.4 & -- & 14.6 & 14.6 
& -- & -- & -- & -- \\
BAR       
& \textcolor{red}{34.1} & $\sim 10^6$ & 8100 & $\sim 10^6$ 
& \textcolor{red}{59.1} & 57.6 & 14.6 & 72.2 
& 68.1 & 38.4 & 9.7 & 48.1 \\
BlackVIP  
& \textcolor{red}{39.4} & $\sim 10^7$ & 8100 & $\sim 10^7$ 
& \textcolor{customgreen}{60.1} & 86.4 & 14.6 & 101.0 
& 72.1 & 57.6 & 9.7 & 67.3 \\
\rowcolor{gray!50} 
\textbf{AReS}     
& \textbf{\textcolor{customgreen}{73.1}} & 160 & 0 & \textbf{160} 
& \textbf{\textcolor{customgreen}{87.2}} & 0.3 & 0 & \textbf{0.3} 
& \textbf{83.2} & 0.2 & 0 & \textbf{0.2} \\
\bottomrule
\end{tabular}
}
\vspace{-2mm}
\end{table*}

\noindent\textbf{Results on Vision Models (VMs).}
\label{subsec:res_vm}
We further validate our approach using ViT-B/16 as the closed-box service model, paired with ViT-B/32 and RN50 local encoders (details in Table~\ref{tab:vitb16_res}; results using an RN101 service model are in Appendix Table~\ref{tab:rn101_res}). Our method consistently outperforms BAR and BlackVIP across all configurations. As shown in the main results, we achieve an average accuracy of 50.4\% (with a ViT-B/32 local encoder), compared to significantly lower accuracy for BAR (20.4\%) and BlackVIP (34.8\%). This superior performance comes with drastically reduced resource requirements; in the full-shot setting, our method uses less than 0.02\% of the API calls and only 3-7\% of the computation time required by competitors like BlackVIP, which demands over 2500M API calls and 750 hours of training.

To further analyze the effectiveness of our priming stage, we compare AReS against a direct glass-box VR baseline performed on the local model. Surprisingly, simply reprogramming the local encoder directly already outperforms the complex ZOO-based BlackVIP (e.g., 45.3\% vs. 34.8\% for the ViT-B/32 local model in Appendix Table~\ref{tab:vm_local_vr_fullshot}). Our priming process then further elevates this strong baseline by a significant margin (e.g., an additional $+5.1\%$ to 50.4\%). Moreover, our approach often exhibits synergistic effects, surpassing both the local VR baseline and the estimated glass-box performance of the service model on several datasets. For instance, with an RN101 service and ViT-B/32 local model, on Flowers102 our method achieves 51.3\% (vs. 42.3\% service glass-box vs. 38.7\% local VR). A detailed analysis is in Appendix Sec.~\ref {sec:app_synergy}. This demonstrates our method's potential to effectively combine knowledge from both models for enhanced reprogramming outcomes with high efficiency.

\noindent\textbf{Results on MLLMs and Real-World APIs.}
To validate AReS’s practical effectiveness and efficiency, we conducted a targeted evaluation on modern services using the EuroSAT 16-shot benchmark. We test AReS against baselines on three distinct models: the open-source MLLM LLaVA, the proprietary API GPT-4o, and the commercial VM API Clarifai. For the commercial APIs, we report the real-world cost in USD, for which glass-box baselines are unavailable. AReS demonstrates substantial gains in both accuracy and efficiency in Table~\ref{tab:mllm_api}, confirming its value in real-world settings.

AReS's strong performance stems from its conceptual strategy shift, which circumvents challenges faced by ZOO-based methods on modern models. On LLaVA, for example, AReS avoids using input perturbations that appear to disrupt the model's delicate vision-language alignment, a problem that causes BAR and BlackVIP's performance to drop below the Zero-shot baseline. Similarly, on the robust GPT-4o, where ZOO-based methods struggle to find effective perturbations, AReS's local reprogramming remains highly effective, achieving a remarkable +27.8\% gain over Zero-shot. On the commercial Clarifai API, a \emph{closed-vocabulary} model where zero-shot is not an option, AReS again achieves the highest accuracy (83.2\%) at a fraction of the cost (\$0.2), making it over \emph{300}x \emph{cheaper} than BlackVIP (\$67.3). These results confirm that AReS is a robust, effective, and economically viable solution for adapting modern closed-box models. More details in the Appendix Sec.~\ref{app:mllm_gpt}.

\begin{figure*}[t]
    \centering
    \subfloat[Data Fraction]{
      \begin{minipage}[b]{0.24\linewidth} 
        \centering  
        \includegraphics[width=\linewidth]{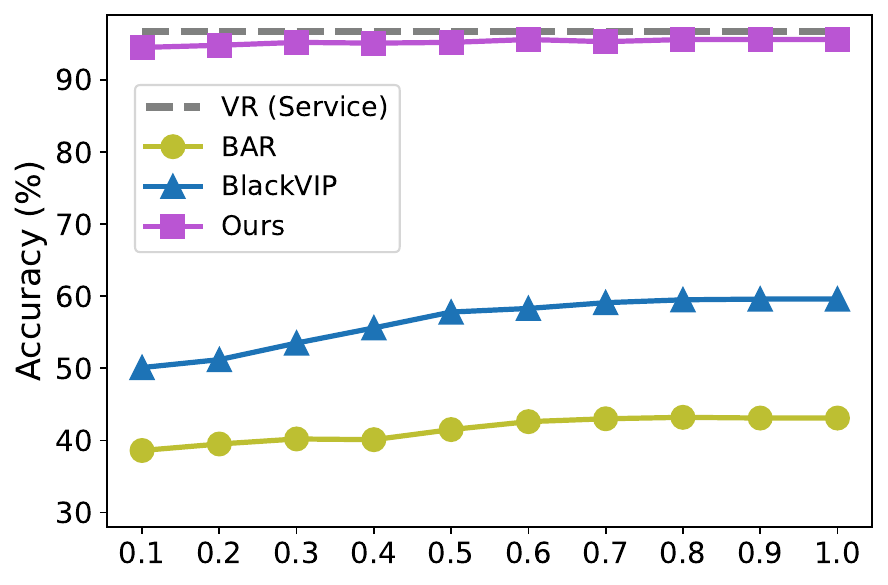} 
        \vspace{-4.5mm}
        \label{fig:ablation_ratio}
      \end{minipage}
    }
    \subfloat[Few-Shot]{
      \begin{minipage}[b]{0.24\linewidth} 
        \centering  
        \includegraphics[width=\linewidth]{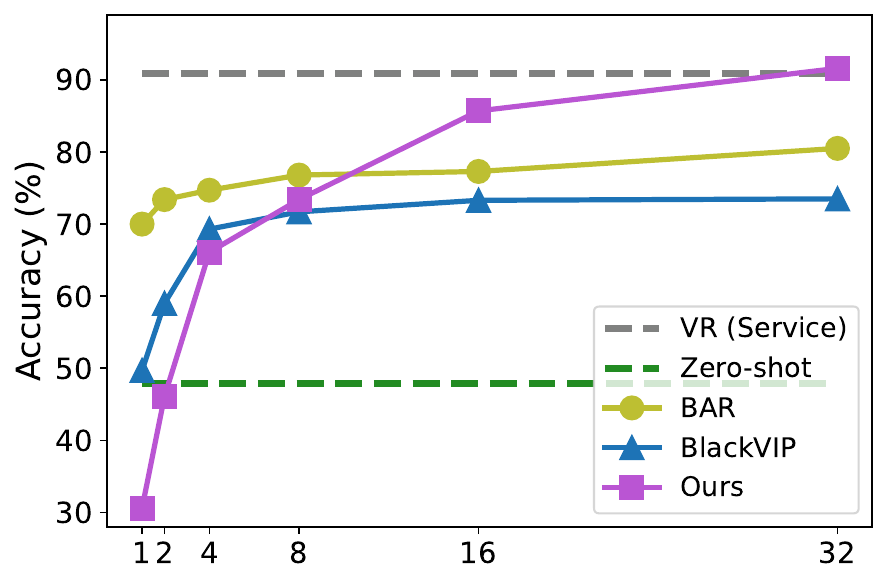} 
        \vspace{-4.5mm}
        \label{fig:ablation_shot}
      \end{minipage}%
    }
    \subfloat[Priming Loss]{
      \begin{minipage}[b]{0.24\linewidth} 
        \centering  
        \includegraphics[width=\linewidth]{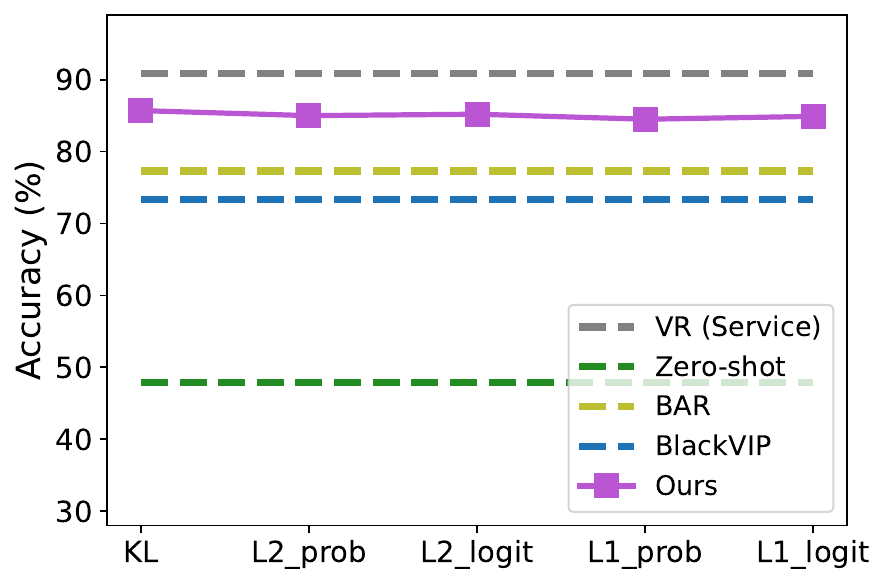}   
        \vspace{-4.5mm}
        \label{fig:ablation_ditil_loss}
      \end{minipage}
    }
    \subfloat[Extra Unlabeled Data]{
      \begin{minipage}[b]{0.24\linewidth} 
        \centering  
        \includegraphics[width=\linewidth]{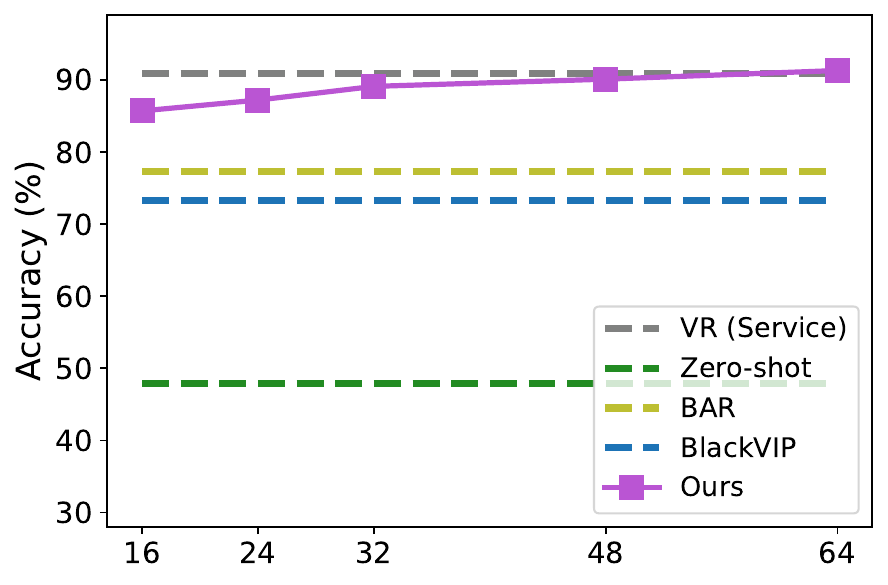}   
        \vspace{-4.5mm}
        \label{fig:ablation_extra_distil}
      \end{minipage}
    }
    \vspace{-1.5mm}
    \caption{Ablation studies for AReS on the EuroSAT dataset: impact of (a) training data fraction (for VMs), (b) few-shot sample size (for VLMs), (c) priming loss function, and (d) amount of extra unlabeled data for priming on accuracy (\%). Default setting involves 16-shot learning, CLIP ViT-B/16 Service, and ViT-B/16 Local encoder unless otherwise indicated by the subplot analysis.}
    \vspace{-2mm}
\end{figure*}

\subsection{Ablation Study}\label{sec:ablation}
We present ablation studies in this section. All experiments use EuroSAT with 16 shots per class, CLIP ViT-B/16 (Service) and ViT-B/16 (Local), unless specified otherwise.

\begin{table}[t]
\centering
\caption{Ablation study across Service/Local model backbones.}
\vspace{-2mm}
\scriptsize
\label{tab:ablation_backbone}
\begin{tabular}{l|cccc|c}
\toprule
\textbf{Method} & \multicolumn{4}{c|}{\textbf{Service Backbone}} &  \\
 & RN50 & RN101 & ViT-B/32 & ViT-B/16 &  Avg.\\
\midrule
ZS & 37.5 & 32.6 & 45.2 & 47.9 & 40.8 \\
BAR & 26.9 & 33.5 & 70.3 & 77.2 & 52.0 \\
\midrule
BlackVIP (RN50) & 51.3 & 50.8 & 62.9 & 68.5 & 58.4 \\
\rowcolor{gray!30}
\textbf{AReS} (RN50) &  {\textbf{62.2}} &  {\textbf{68.5}} &  {\textbf{63.2}} &  {\textbf{78.6}} &  {\textbf{68.1}} \\
\midrule
BlackVIP (ViT-B/16) & 48.4 & 51.3 & 67.9 & 73.3 & 60.2 \\
\rowcolor{gray!30}
\textbf{AReS} (ViT-B/16) &  {\textbf{81.8}} &  {\textbf{81.9}} &  {\textbf{82.0}} &  {\textbf{85.7}} & \textbf{82.9} \\
\bottomrule
\end{tabular}
\vspace{-2mm}
\end{table}

\noindent\textbf{Backbone Versatility.} We examine the effect of varying service model backbones and local encoder choices to demonstrate our approach's versatility when the service model is VLM.
The results in Table~\ref{tab:ablation_backbone} indicate that our method, whether using a ViT-B/16 or an RN50 local encoder, consistently outperforms baselines like BAR and BlackVIP across diverse CLIP backbones (RN50, RN101, ViT-B/32, and ViT-B/16) used in VLM contexts. For instance, averaged across these service backbones, our method with a ViT-B/16 local encoder achieves an average accuracy of 82.9\%, substantially higher than BlackVIP with the same local encoder (60.2\%). This underscores the robustness of our method across different architectural choices.

\noindent\textbf{Component Analysis.} We dissect the contributions of our core components in Table~\ref{tab:ablation_component}: Priming, Visual Reprogramming (VR). We observe that priming the local encoder alone ("Primed") yields 45.6\%, which is slightly below the 47.9\% zero-shot baseline. It is important to note that this Primed result is not always directly comparable; the priming occurs in the service model's label space (e.g., ImageNet), which can be disjoint from the target task's labels (e.g., EuroSAT). This preparatory focus, which does not require a shared label space, fundamentally distinguishes our approach from conventional knowledge distillation. While the Primed model's accuracy is lower than the zero-shot baseline, an expected outcome, as its performance is bounded by the service model, the true goal is to prepare the local model for reprogramming. Our final AReS result (85.7\%) demonstrates that the local model's reprogrammability is significantly enhanced by this step. 
Local adaptation methods without priming show varied success, with Local VR (70.6\%), Local LP (73.8\%), and their combination ``Local VR+LP" (80.1\%) all improving on the baseline. However, our full AReS approach, which combines Priming with Reprogramming, achieves the highest accuracy of 85.7\%. 
This result demonstrates that both components are crucial, outperforming all local-only tuning baselines, and confirms that the priming stage synergistically prepares the local model for a more effective reprogramming.

\begin{table}[t]
\centering
\caption{Component analysis of Priming, Visual Reprogramming (VR), 
and Linear Probing (LP).}
\label{tab:ablation_component}
\vspace{-2mm}
\scriptsize
\resizebox{\columnwidth}{!}{
\begin{tabularx}{\columnwidth}{X|YcY|Y}
\toprule
\textbf{Method} & \textbf{Priming} & \textbf{VR} & \textbf{LP} & \textbf{Accuracy} \\
\midrule
Zero-shot        & --      & --      & --      & 47.9 \\
Primed               & \cmark  & --      & --      & 45.6 \\
Local VR         & --      & \cmark  & --      & 70.6 \\
Local LP         & --      & --      & \cmark  & 73.8 \\
Local VR+LP      & --      & \cmark  & \cmark  & 80.1 \\
\midrule
\rowcolor{gray!30}
\textbf{AReS (Ours)} & \cmark & \cmark & -- & \textbf{85.7} \\
\bottomrule

\end{tabularx}
}
\vspace{-5mm}
\end{table}
\noindent\textbf{Impact of Training Data Size.} The impact of training data size is investigated for both VMs and VLMs. For VMs, Fig.~\ref{fig:ablation_ratio} shows our method consistently outperforms BAR and BlackVIP, maintaining significant performance even when data is scarce. More VM evaluation in the 16-shot setting is provided in Appendix Sec.~\ref{app:vm_fewshot}. For VLMs, varying the number of samples per class in a few-shot setting is presented in Fig.~\ref{fig:ablation_shot}. While at very low shot counts (e.g., $<8$ shots), the performance of our method can be comparable or slightly underperform some baselines due to suboptimal priming and local VR training from extreme data scarcity, it stabilizes quickly. From 8 shots onwards, our method consistently and significantly outperforms baselines, and with more data (e.g., 32-shots), its performance can be further improved to levels comparable to glass-box service model VR.

\noindent\textbf{Priming Strategies.} We explore different priming loss functions and the utility of extra unlabeled data, with results in Fig.~\ref{fig:ablation_ditil_loss} and Fig.~\ref{fig:ablation_extra_distil} respectively. Among various priming losses (KL divergence, L1/L2 losses on probabilities or logits), KL divergence, used in our main experiments, provides strong and stable performance. While we operate under the strict closed-box assumption (probabilities only), logits-based losses are included for broader comparative analysis. Furthermore, a key benefit of priming is its lack of requirement for extra labeled data, which allows our method to leverage unlabeled downstream task data. As shown in Fig.~\ref{fig:ablation_extra_distil}, increasing the amount of unlabeled data for the priming stage (while keeping VR training fixed at 16-shot) steadily improves our method's performance. This ability to benefit from extra unlabeled data is a unique advantage of our approach compared to BMR baselines (BAR and BlackVIP), which cannot utilize such data in this manner.

\section{Conclusion}
\label{sec:conclusion}
In this work, we addressed key challenges in adapting closed-box Model-as-a-Service APIs, including the heavy query costs, slow optimization, and diminishing effectiveness of ZOO-based methods on modern service models. We introduced AReS, an alternative reprogramming strategy that uses a single-pass interaction with the service API to prepare a local encoder, enabling efficient glass-box reprogramming without further API dependence. This simple, theoretically-supported design yields substantial gains in both performance and practicality.
Across diverse benchmarks, AReS consistently outperforms state-of-the-art ZOO approaches on VMs and VLMs, and delivers large improvements on proprietary APIs such as GPT-4o while ZOO-based approaches falter.
By dramatically reducing API calls by over 99.99\%, AReS enables cost-free inference and offline deployment. These comprehensive findings underscore AReS as a robust, economical, and highly effective solution for closed-box model adaptation. We show that reprogrammability can be efficiently unlocked in a local model through a minimal, one-time priming interaction, facilitating a more practical path for transfer learning in API-centric ecosystems.

\clearpage

\section*{Acknowledgment}
We thank the anonymous reviewers for their insightful comments and suggestions. 
This work was supported in part by the NSF EPSCoR-Louisiana Materials Design Alliance (LAMDA) program \#OIA-1946231.

{
    \small
    \bibliographystyle{ieeenat_fullname}
    \bibliography{main}
}


\clearpage
\onecolumn
\begin{center}
    \huge \textbf{Prime Once, then Reprogram Locally:
An Efficient Alternative to Black-Box Service Model Adaptation  \texttt{Supplementary Materials}}
\end{center}

\noindent This appendix provides comprehensive details supporting the main paper. Section \ref{app:mr_detail} elaborates on model reprogramming techniques, while Section \ref{app:exp_setting} presents detailed experimental configurations, including dataset statistics and implementation specifics for all baselines and our AReS method (Section \ref{app:implementation_detail}). Section \ref{app:theoretical_analysis} details the complete theoretical analysis establishing performance bounds between service and local models. Additional experimental results are reported in Section \ref{app:exp_result}, covering few-shot VM experiments (Section \ref{app:vm_fewshot}), local model enhancement analysis (Section \ref{app:local_wb_vr}), evaluations on \textbf{real-world proprietary APIs} such as \footnote{\texttt{https://platform.\,openai.com/docs/models/gpt-4o}}{GPT-4o} and \footnote{\texttt{https://www.\,clarifai.com/}}{Clarifai} (Section \ref{app:mllm_gpt}), and a detailed analysis of challenging VLM scenarios (Section \ref{app:failure_analysis}). Section \ref{app:further_discussion} offers further discussion on our contributions and design rationale. Finally, Section \ref{app:limitations} discusses the limitations of our approach.

\section{Details of Model Reprogramming}
\label{app:mr_detail}
Model Reprogramming (MR) \citep{chen2024model, jia2022visual, bahng2022exploring, zhang2022fairness, xiao2026promptbased} is a technique that adapts pre-trained models, such as a source model $\mathcal{F}_S$, to new tasks without altering their underlying parameters. This is particularly useful when dealing with powerful models accessed as a service (MaaS) where direct modification is not possible. The core idea involves an input transformation function $g_{\mathrm{in}}$ that adapts inputs $x^T$ from the downstream task domain $\mathcal{X}^T$ to the source domain $\mathcal{X}^S$ through a learnable prompt $\mathbf{P}$. Additionally, an output or label mapping function $g_{\mathrm{out}}$ is employed to bridge the source label space $\mathcal{Y}^S$ and the target label space $\mathcal{Y}^T$. This approach aims to leverage the rich features learned by large pre-trained models for new, potentially resource-scarce tasks, by essentially ``tricking" the model into performing a different function through these carefully crafted input and output manipulations.

\subsection{Input Transformation}
Input transformation \citep{tsao2024autovp, zhang2022fairness, balckvip_CVPR23, bar_icml20} in Model Reprogramming, denoted as $g_{\mathrm{in}}(x^T|\mathbf{P})$, serves to bridge the domain gap between the downstream task and the source task the original model $\mathcal{F}_S$ was trained on. It introduces a learnable prompt or program $\mathbf{P}$ that modifies the downstream task inputs $x^T$ before they are fed to $\mathcal{F}_S$, aiming to make the downstream data compatible with its input expectations. Common approaches in Visual Reprogramming (VR) include padding-based VR, which adds trainable noise patterns to the outer frames of an image while preserving its integrity, and watermarking-based VR, where trainable noise patterns are overlaid directly onto the input images. In the context of BAR \citep{bar_icml20}, the input transformation takes the form of an ``adversarial program'' $P = \tanh(W \odot M)$, where $W$ represents learnable parameters and $M$ is a binary mask ensuring that the original embedded target data remains unchanged. This adversarial program $P$ is universal to all target data samples and is learned to make the model $\mathcal{F}_S$ produce outputs that can be mapped to the desired target task labels. The learning of these parameters, especially in closed-box settings, often relies on zeroth-order optimization techniques.

\subsection{Output Mapping}
\label{app:output_mapping}
Output mapping \citep{chen2023understanding, cai2024bayesian, bar_icml20, kloberdanz2021improved}, $g_{\mathrm{out}}$, is a crucial step that translates the predictions from the pre-trained model's original label space $\mathcal{Y}^S$ to the downstream task's label space $\mathcal{Y}^T$, which are often distinct. This process is frequently designed to be gradient-free, meaning it does not introduce additional trainable parameters requiring backpropagation, thus preserving efficiency, especially for tasks with large label spaces. Existing gradient-free methods like Random Label Mapping (RLM) \citep{elsayed2018adversarial}, Frequent Label Mapping (FLM) \citep{bar_icml20}, and Iterative Label Mapping (ILM) \citep{chen2023understanding} typically establish a one-to-one correspondence. However, such one-to-one mappings can be limiting as they may overlook more complex, potentially many-to-many relationships between pre-trained and downstream labels. To address this, probabilistic and multi-label mapping strategies have been developed. Bayesian-guided Label Mapping (BLM), for example, constructs a probabilistic matrix quantifying pairwise relationships between pre-trained and downstream labels using Bayesian conditional probability, allowing for a flexible many-to-many mapping. 
Similarly, BAR \citep{bar_icml20} can utilize multi-label mapping where multiple source labels map to a single target label, often determined by a frequency-based scheme from initial predictions. 

The nature of $g_{\mathrm{out}}$ can differ for VMs and VLMs. For VMs, if label spaces differ, an explicit $g_{\mathrm{out}}$ like BLM is necessary, and in the AReS framework, priming for VMs occurs in the service model's label probability space. For VLMs, $g_{\mathrm{out}}$ might be an identity mapping if label spaces align, or the VLM's text encoder can naturally transform outputs to the downstream task label space.

\begin{figure}[!t]
    \centering
    \includegraphics[width=1\linewidth]{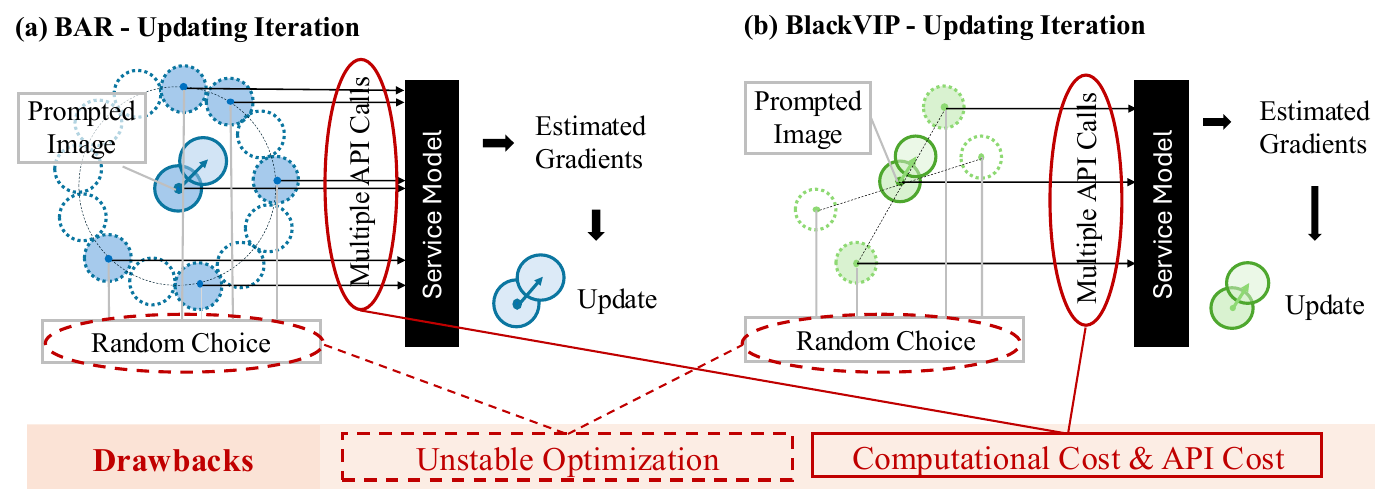}
    \caption{Illustration of Zeroth-Order Optimization (ZOO) techniques commonly used in closed-box model reprogramming. \textbf{(a)} BAR with Randomized Gradient-Free (RGF) method estimates gradients by querying the model with random directional perturbations. \textbf{(b)} BlackVIP with Simultaneous Perturbation Stochastic Approximation with Gradient Correction (SPSA-GC) approximates gradients using only two model queries with a randomly generated perturbation vector. Both approaches suffer from high query complexity and noisy gradient estimates, leading to unstable and computationally intensive optimization, especially for high-dimensional prompts.}
    \label{fig:drawback}
\end{figure}    

\subsection{Training}
The training process in model reprogramming focuses on learning the parameters of the input transformation $g_{\mathrm{in}}(\cdot|P)$, such as the adversarial program $P$, while the pre-trained source model $\mathcal{F}_S$ remains frozen. The objective is to minimize a loss function on the downstream task data using the transformed inputs and mapped outputs. In a glass-box setting, where gradients from $\mathcal{F}_S$ are accessible, standard gradient-based optimization can efficiently update the transformation parameters. Our AReS method leverages this by first priming a local, fully accessible model $\mathcal{F}_L$, then performing efficient glass-box prompt learning.
Conversely, when $\mathcal{F}_S$ is a closed-box, Zeroth-Order Optimization (ZOO) techniques are employed. These methods estimate gradients using only the model's input-output responses, for instance, through random gradient-free (RGF) optimization \citep{tu2019autozoom, liu2018zeroth}, Simultaneous Perturbation Stochastic Approximation with Gradient Correction (SPSA-GC) \citep{spall1992multivariate, spall1997one}, as shown in Fig.~\ref{fig:drawback} and Fig.~\ref{fig:opt_comparison}, closed-box training presents significant difficulties: optimization can be unstable due to noisy gradient estimates; 
it incurs high computational costs and numerous API calls; it requires continuous API dependency for both training and inference; and the performance gains are often uncertain despite the substantial investment. Our AReS approach aims to circumvent these challenges by shifting the reprogramming to a locally primed model.

\section{Experimental Setting}
\label{app:exp_setting}
\subsection{Datasets}
\label{app:dataset}

\textbf{Datasets.} We evaluate our method on ten diverse datasets commonly used in transfer learning literature. These datasets span various domains including fine-grained categorization (Flowers102 \citep{flowers102}, StanfordCars \citep{stanfordcars}), texture recognition (DTD \citep{DTD}), action recognition (UCF101 \citep{UCF101}), food classification (Food101 \citep{Food101}), traffic sign recognition (GTSRB \citep{GTSRB}), satellite imagery (EuroSAT \citep{EuroSAT}), animal recognition (OxfordPets \citep{Oxfordpets}), scene classification (SUN397 \citep{sun397}), and digit recognition (SVHN \citep{SVHN}). Following established protocols in prior work \citep{balckvip_CVPR23}, we adopt a few-shot learning setup with 16 randomly selected training examples per class, using the entire test set for evaluation when the service model is VLM. And full-shot learning is applied when the service model is VM, consistent with \citep{bar_icml20}. Table~\ref{tab:dataset} provides detailed statistics for each dataset.

\begin{table}[ht]
\caption{Detailed dataset information.}
\small
\begin{center}
\begin{tabularx}{\textwidth}{l|ZZZ|r}
\toprule
\multicolumn{1}{l|}{Dataset} & Full-shot Training & 16-shot Training & Testing Set Size & Number of Classes \\
\midrule
Flowers102     & 4,093  & 1,632  & 2,463  & 102 \\
DTD            & 2,820  & 752    & 1,692  & 47  \\
UCF101         & 7,639  & 1,616  & 3,783  & 101 \\
Food101        & 50,500 & 1,616  & 30,300 & 101 \\
GTSRB          & 39,209 & 688    & 12,630 & 43  \\
EuroSAT        & 13,500 & 160    & 8,100  & 10  \\
OxfordPets     & 2,944  & 592    & 3,669  & 37  \\
StanfordCars   & 6,509  & 3,136  & 8,041  & 196 \\
SUN397         & 15,888 & 6,352  & 19,850 & 397 \\
SVHN           & 73,257 & 160    & 26,032 & 10  \\
\bottomrule
\end{tabularx}
\end{center}
\label{tab:dataset}
\end{table}

\subsection{Experimental Scope and Rationale}
\label{app:scope}
Our focus on image classification is a deliberate choice made to ensure a direct and rigorous comparison with the established art in closed-box Model Reprogramming (BMR)~\citep{bar_icml20, balckvip_CVPR23}. The primary methods in this field, including our main baselines (BAR~\citep{bar_icml20}, BlackVIP~\citep{balckvip_CVPR23}, LLM-Opt~\citep{liu2024language}) and other related works~\citep{wang2024craft, ouali2023black, guo-etal-2023-black, kunananthaseelan2024lavip, sun_ijcai22} shown in Table~\ref{tab:bb_comparison}, are all benchmarked on classification tasks for both standard Vision Models (VMs) and Vision-Language Models (VLMs). Adhering to this established protocol allows for a fair and robust evaluation of our framework, which is designed to provide an efficient reprogramming solution for both model types. Furthermore, classification remains a challenging and relevant testbed; recent work has shown that powerful Multimodal Large Language Models (MLLMs) e.g., LLaVA~\citep{llava} often suffer from catastrophic forgetting, failing to retain the full classification performance of their underlying visual towers~\citep{zhai2023investigating}. This highlights the non-trivial nature of adapting these models for pure classification and validates the importance of our approach.

While extending AReS to generative tasks like Visual Question Answering (VQA) is a valuable direction for future research, such tasks are outside the scope of this initial investigation. A core goal of our work is to support standard vision classifiers like ViT and ResNet, for which open-ended generative tasks are not directly applicable. By demonstrating AReS's effectiveness on the fundamental task of classification---a common ground for both VMs and VLMs---we build a strong and reliable foundation for future extensions into more complex, multimodal reasoning tasks.

\subsection{Baselines and Implementation Details}
\label{app:implementation_detail}
Our experiments compare AReS against three primary closed-box visual reprogramming baselines: BAR, BlackVIP, and LLM-Opt.

\noindent\textbf{BAR~\citep{bar_icml20}}\footnote{\url{https://github.com/yunyuntsai/black-box-Adversarial-Reprogramming}} repurposes closed-box models by learning a universal adversarial program that is added to target inputs. It relies on Randomized Gradient-Free (RGF) optimization, based on input-output responses, and employs a multi-label mapping scheme. For our VMs experiments, we utilize BLM\footnote{\label{fn:blm_code}\url{https://github.com/tmlr-group/BayesianLM}} for BAR to ensure a fair comparison with other methods, including our own, which also uses BLM. We implemented BAR by referencing its official codebase and BlackVIP's reported BAR implementation, adhering to the hyperparameters detailed in BlackVIP's appendix\footnote{\label{fn:blackvip_config}\url{https://github.com/changdaeoh/BlackVIP/blob/main/docs/configuration.md}} for stable convergence. BAR uses focal loss as its learning objective.

\noindent\textbf{BlackVIP \citep{balckvip_CVPR23}}\footnote{\label{fn:blackvip_code}\url{https://github.com/changdaeoh/BlackVIP}} enhances closed-box visual prompting by introducing a Coordinator network to generate input-dependent visual prompts and employs Simultaneous Perturbation Stochastic Approximation with Gradient Correction (SPSA-GC) for optimization. We use the official BlackVIP codebase and strictly adhere to the hyperparameter settings reported in their paper and detailed in their configuration files, which cover learning rates and SPSA-GC specific parameters. BlackVIP utilizes cross-entropy loss. For adapting VLMs, all methods, including BlackVIP and BAR, are evaluated in a 16-shot learning setting (16 training samples per class), consistent with the setup in the BlackVIP paper. 

\noindent\textbf{LLM-Opt~\citep{liu2024language}}\footnote{\url{https://llm-can-optimize-vlm.github.io}} utilizes a large language model (e.g., GPT-4) as a closed-box optimizer to find effective text prompts for a target VLM. The method employs an automated "hill-climbing" procedure, where it provides the LLM optimizer with both high- and low-performing prompts as textual feedback to guide the search for better candidates \citep{liu2024language}. However, this approach has two significant limitations. First, it introduces a second, costly API dependency for the LLM optimizer, which can amount to over \$100 in API fees for a single experimental run \citep{liu2024language}. Second, because the method exclusively optimizes the text prompt, its applicability is restricted to VLMs and it cannot be used to adapt standard VMs, which are a key focus of our work.

\noindent \textbf{AReS}, our proposed method, the visual reprogramming (VR) components are implemented leveraging the publicly available codebase of BLM. The priming stage, transferring from the service API to the local model's classification head, is performed for 100 epochs. The subsequent glass-box VR on the primed local model is conducted for 200 epochs. Both stages use the Adam optimizer with learning rates of 0.001 for KD and 0.01 for local VR, respectively. The local VR part of AReS employs cross-entropy loss. AReS follows the 16-shot protocol for VLM experiments and operates in a full-shot setting for VM experiments. Following \citep{balckvip_CVPR23}, pre-trained Vision Transformer (ViT) models and encoders used in our experiments are sourced from \texttt{timm} \citep{timm}, ResNet models are from \texttt{PyTorch} \citep{pytorch}, and CLIP models are from the official \texttt{CLIP} repository \citep{CLIP}.

\subsection{Clarification of VLM and VM Evaluation Setups}

Our experiments distinguish between two primary evaluation setups: one for Vision-Language Models (VLMs) and another for standard Vision Models (VMs). The core differences lie in how each model type handles the output space of the downstream task and the data settings used for evaluation, which were chosen to ensure fair and direct comparisons with established baselines.

\noindent\textbf{VLM Evaluation Setup.}
In our experiments, the VLM setup primarily involves models like CLIP, which is consistent with the standard set by our main baseline, BlackVIP. The key characteristics are:

\begin{itemize}
    \item \textbf{Output Handling:} VLMs utilize a text encoder to perform zero-shot classification without a fixed output head. This means their output space is naturally aligned with the downstream task's text labels. Consequently, a separate or complex label mapping mechanism is not required, simplifying the adaptation process.
    \item \textbf{Data Setting:} To ensure a direct comparison, the evaluation follows BlackVIP's established protocol, which uses a \textbf{16-shot} learning setting. The closed-box model reprogramming for all compared methods is learned using only these few-shot labeled samples.
\end{itemize}

\noindent\textbf{VM Evaluation Setup.}
The VM setup uses standard, pre-trained vision models with a fixed output vocabulary, such as an ImageNet-pretrained ViT-B/16 or ResNet101. This setup presents a more significant challenge:

\begin{itemize}
    \item \textbf{The Challenge of Label Space Mismatch:} Unlike VLMs, these models face a fundamental label space mismatch. Their fixed output vocabulary (e.g., 1000 ImageNet classes) does not align with the labels of the downstream task (e.g., Flowers102). This is a "non-trivial problem" that prior work like BlackVIP explicitly avoided.
    \item \textbf{AReS's Solution:} Our framework directly solves this challenge by incorporating Bayesian-guided Label Mapping (BLM) to bridge the source and target label spaces. In this setup, AReS's priming stage occurs within the service model's source label space, as it has no inherent knowledge of the target class names.
    \item \textbf{Data Setting:} Following the standard for BMR on vision models established by BAR, these experiments are conducted in a \textbf{full-shot} setting. For completeness, a few-shot VM setting was also explored in our appendix.
\end{itemize}

\begin{figure}
    \centering
    \includegraphics[width=1\linewidth]{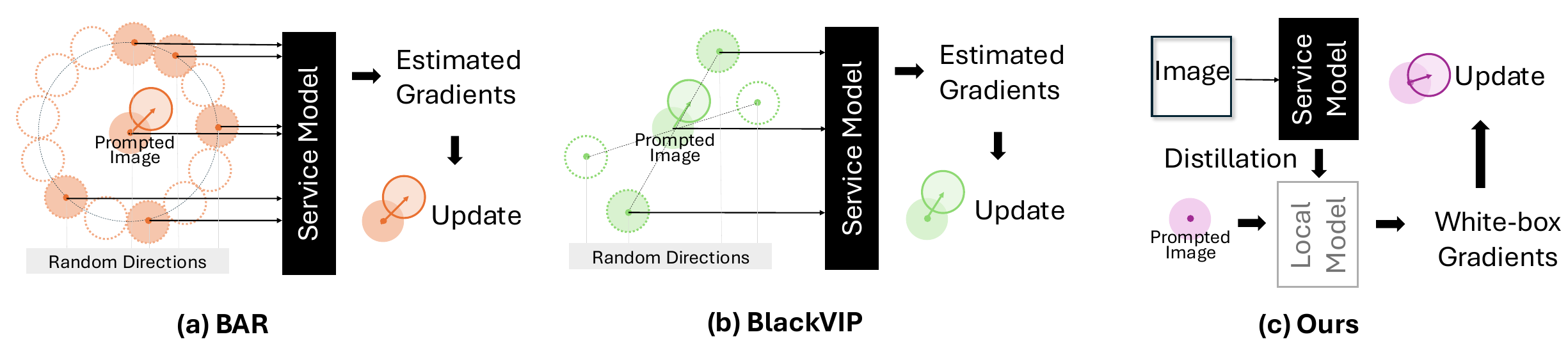}
    \caption{Comparison of closed-box visual reprogramming approaches. \textbf{(a)} BAR and \textbf{(b)} BlackVIP rely on Zeroth-Order Optimization (ZOO) by repeatedly querying the service model with perturbed inputs (e.g., using random directions) to estimate gradients for updating the visual prompt. These methods suffer from high API call costs and potentially unstable optimization. \textbf{(c)} Our AReS method performs a one-time priming from the service model to a local model. Subsequent visual prompt optimization occurs efficiently on this local model using glass-box gradients, eliminating further API calls and enabling stable, cost-effective adaptation.}
    \label{fig:opt_comparison}
\end{figure}

\section{Details of Theoretical Analysis}
\label{app:theoretical_analysis}
In this section, we provide a theoretical analysis to understand the effectiveness of our proposed method, AReS. Our framework involves two primary components affecting the final performance on the downstream task: priming from the service model $\mathcal{F}_S$ to the local encoder $\mathcal{F}_L$, and visual reprogramming (VR) applied to the local model using prompt $\mathbf{P}$. Our analysis is related to the broader study of representation transferability via task-relatedness~\citep{NEURIPS2024_d3602fc9}, which provides formal tools for understanding when and why transferring representations between models is effective. We begin our analysis by considering the scenario where the model's output logits are directly aligned with the target labels, effectively assuming the label mapping $g_{\mathrm{out}}$ is an identity function due to the text encoder.

\noindent\textbf{Notations:}
Let $\mathcal{D}^T$ represent the downstream task data distribution $p(x,y)$ over inputs $x \in \mathcal{X}^T$ and labels $y \in \mathcal{Y}^T$.
Let $\mathcal{F}_S: \mathcal{X}^S \to \mathbb{R}^{K^S}$ be the closed-box service model and $\mathcal{F}_L: \mathcal{X}^S \to \mathcal{Z}$ be the local surrogate model. After priming, $\mathcal{F}_L$ includes the learned linear layer mapping inputs to the downstream task's logit space $\mathcal{Z} = \mathbb{R}^{K^T}$. The visual reprogramming involves an input transformation $g_{\mathrm{in}}(x, p)$ parameterized by $p$. $\ell: \mathcal{Z} \times \mathcal{Y}^T \to \mathbb{R}_{\ge 0}$ is cross-entropy loss.

Let $p_S(x)$ and $p_L(x)$ denote the probability distributions generated by the service model $\mathcal{F}_S$ and the local model $\mathcal{F}_L$ for an input $x$, respectively. 
For the purpose of this theoretical analysis, we introduce logits. Let $z_S(x)$ and $z_L(x)$ be the logits produced by the service model $\mathcal{F}_S$ and the local model $\mathcal{F}_L$ for input $x$ \emph{before} visual reprogramming. It is important to note that this assumption of direct logit access from $\mathcal{F}_S$ is made for theoretical tractability and differs from the practical setting in the main paper, where only probabilities are assumed to be accessible from the service model.
Let $z^*_S(x, \mathbf{Q}^*) = \mathcal{F}_S(g_{\mathrm{in}}(x, \mathbf{Q}^*))$ and $z^*_L(x, \mathbf{P}^*) = \mathcal{F}_L(g_{\mathrm{in}}(x, \mathbf{P}^*))$ be the logits produced by the service model and local model, respectively, for input $x$ \emph{after} visual reprogramming with their respective optimal prompts $\mathbf{Q}^*$ and $\mathbf{P}^*$.

The downstream risk for the local model $\mathcal{F}_L$ \emph{without} VR is:
        $$ \mathcal{R}_L(\mathcal{D}^T) := \mathbb{E}_{(x,y) \sim \mathcal{D}^T}[\ell(z_L(x), y)] 
        $$
The downstream risk for the local model $\mathcal{F}_L$ \emph{after} applying VR with its optimal prompt $\mathbf{P}^*$ is:
        $$ \mathcal{R}_L(\mathcal{D}^T, \mathbf{P}^*) := \mathbb{E}_{(x,y) \sim \mathcal{D}^T}[\ell(z^*_L(x, \mathbf{P}^*), y)] 
        $$
Similarly, for the service model $\mathcal{F}_S$, the downstream risk \emph{without} VR is:
        $$ \mathcal{R}_S(\mathcal{D}^T) := \mathbb{E}_{(x,y) \sim \mathcal{D}^T}[\ell(z_S(x), y)] \,.$$
And the downstream risk for the service model $\mathcal{F}_S$ \emph{after} applying VR with its optimal prompt $\mathbf{Q}^*$ is:
        $$ \mathcal{R}_S(\mathcal{D}^T, \mathbf{Q}^*) := \mathbb{E}_{(x,y) \sim \mathcal{D}^T}[\ell(z^*_S(x, \mathbf{Q}^*), y)] \,.$$

\noindent\textbf{Definitions and Assumptions.}
We introduce the following definitions and assumptions for our theoretical analysis.
\begin{definition} \label{app:def_faithful_distill}
    ($\epsilon$-Faithful Priming). The priming from $\mathcal{F}_S$ to $\mathcal{F}_L$ is considered $\epsilon$-faithful if the expected $L_1$ norm of the difference between their output logits is bounded by $\epsilon \ge 0$, both before and after applying their respective optimal visual reprogramming prompts:
$$ \mathbb{E}_{(x,y) \sim \mathcal{D}^T}[\|z_S(x) - z_L(x)\|_1] \leq \epsilon$$
and
$$ \mathbb{E}_{(x,y) \sim \mathcal{D}^T}[\|z^*_S(x, \mathbf{Q}^*) - z^*_L(x, \mathbf{P}^*)\|_1] \leq \epsilon \,.$$
This implies that the priming is sufficiently effective such that the logit distributions of the service and local models are closely aligned. A smaller $\epsilon$ indicates a more faithful priming.
\end{definition}

\begin{assumption} \label{app:assum_service_model_superiority}
    (Service Model Superiority). The service model $\mathcal{F}_S$ is inherently more powerful or better aligned with the downstream task distribution $\mathcal{D}^T$ compared to the local model $\mathcal{F}_L$, both before and after optimal visual reprogramming. Formally:
    $$ \mathcal{R}_S(\mathcal{D}^T) \leq \mathcal{R}_L(\mathcal{D}^T) $$
    and
    $$ \mathcal{R}_S(\mathcal{D}^T, \mathbf{Q}^*) \leq \mathcal{R}_L(\mathcal{D}^T, \mathbf{P}^*) \,.$$
    This is a natural assumption that motivates the use of the service model.
\end{assumption}

\begin{assumption} \label{app:assum_faithful_distill}
    ($\epsilon$-Faithful Priming). The priming process is effective, resulting in the local model $\mathcal{F}_L$ closely mimicking the logit distributions of the service model $\mathcal{F}_S$, both immediately after priming (before VR) and after both models have been optimally reprogrammed for the downstream task.
\end{assumption}

\subsection{Service and Local Model Performance Difference Bound}
This section establishes the theoretical foundation for understanding the performance relationship between service and local models in our AReS framework. We begin by proving that the cross-entropy loss function exhibits Lipschitz continuity with respect to logit differences (Lemma \ref{app:lemma_lipschitz_cross_entropy_general}), which provides the mathematical basis for our subsequent analysis. Building on this property, we derive performance difference bounds (Lemma \ref{app:lemma_perf_diff_bound}) that quantify how faithful priming affects the performance gap between models. Finally, we establish tight bounds on service model performance relative to the primed local model (Theorem \ref{app:thm_bound_service_model_perf}), demonstrating that faithful priming enables the local model to closely approximate the service model's capabilities both before and after optimal visual reprogramming.

\begin{lemma} \label{app:lemma_lipschitz_cross_entropy_general}
(Lipschitz Continuity of Cross-Entropy Loss with respect to Logits).
The cross-entropy loss function $\ell(z, y) = -\sum_{j=1}^{K^T} y_j \log(p_j(z))$, where $p_j(z)$ are softmax probabilities derived from logits $z \in \mathbb{R}^{K^T}$ and $y$ is a one-hot true label vector, is Lipschitz continuous with respect to the logits $z$. The specific Lipschitz constant depends on the norm used to measure the difference between logit vectors.
Specifically: $\ell(z,y)$ is 1-Lipschitz with respect to the $L_1$ norm of the logits:
$$|\ell(z_1, y) - \ell(z_2, y)| \leq \|z_1 - z_2\|_1\,,$$
for any two logit vectors $z_1, z_2$. These constants do not explicitly depend on the number of classes $K^T$ (for $K^T \ge 2$).
\end{lemma}

\begin{proof}
The gradient of the cross-entropy loss $\ell(z, y)$ with respect to the logits $z$ is given by:
$$ \nabla_z \ell(z, y) = p(z) - y \,,$$
where $p(z)$ is the vector of softmax probabilities derived from $z$, and $y$ is the one-hot true label vector.
A differentiable function $f(x)$ is $L$-Lipschitz with respect to a norm $\|\cdot\|_p$ if the dual norm $\|\cdot\|_q$ of its gradient is bounded by $L$ (i.e., $\|\nabla f(x)\|_q \leq L$), by Hölder's inequality.
For the $L_1$ norm of logits ($p=1$, dual norm $q=\infty$):
We examine the $L_\infty$ norm of the gradient:
$$ \|\nabla_z \ell(z, y)\|_{\infty} = \|p(z) - y\|_{\infty} = \max_{j=1, \dots, K^T} |p_j(z) - y_j| \,.$$
Let $k$ be the index of the true class ($y_k=1$, and $y_j=0$ for $j \neq k$).
Then $|p_k(z) - 1| = 1 - p_k(z) \leq 1$. For $j \neq k$, $|p_j(z) - 0| = p_j(z) \leq 1$.
Thus, $\|\nabla_z \ell(z, y)\|_{\infty} \leq 1$. This establishes $L=1$.

\end{proof}

\begin{lemma} \label{app:lemma_perf_diff_bound} 
(Performance Difference Bound due to Priming Faithfulness).
Under Assumption \ref{app:assum_faithful_distill} ($\epsilon$-Faithful Priming), the difference in performance between the local model $\mathcal{F}_L$ and the service model $\mathcal{F}_S$ is bounded by $\epsilon$ as follows:
\begin{enumerate}[leftmargin=3em]
    \item Before visual reprogramming:
    $$ \mathcal{R}_L(\mathcal{D}^T) - \mathcal{R}_S(\mathcal{D}^T) \leq \epsilon \,;$$
    \item After optimal visual reprogramming with prompts $\mathbf{P}^*$ for $\mathcal{F}_L$ and $\mathbf{Q}^*$ for $\mathcal{F}_S$:
    $$ \mathcal{R}_L(\mathcal{D}^T, \mathbf{P}^*) - \mathcal{R}_S(\mathcal{D}^T, \mathbf{Q}^*) \leq \epsilon \,.$$
\end{enumerate}
\end{lemma}

\begin{proof}
We will prove the first part of the lemma regarding performance before visual reprogramming. The proof for the second part (after optimal visual reprogramming) follows an identical structure, applying the arguments to $z^*_L(x, \mathbf{P}^*)$ and $z^*_S(x, \mathbf{Q}^*)$ and using the corresponding condition from Assumption \ref{app:assum_faithful_distill}.

The difference in downstream risks between the local model $\mathcal{F}_L$ and the service model $\mathcal{F}_S$ is:
\begin{align*} 
\mathcal{R}_L(\mathcal{D}^T) - \mathcal{R}_S(\mathcal{D}^T) &= \mathbb{E}_{(x,y) \sim \mathcal{D}^T}[\ell(z_L(x), y)] - \mathbb{E}_{(x,y) \sim \mathcal{D}^T}[\ell(z_S(x), y)] \\ &= \mathbb{E}_{(x,y) \sim \mathcal{D}^T}[\ell(z_L(x), y) - \ell(z_S(x), y)]\,.
\end{align*}
By Lemma \ref{app:lemma_lipschitz_cross_entropy_general}, the cross-entropy loss $\ell(z,y)$ is 1-Lipschitz continuous with respect to the $L_1$ norm of its logit inputs. Therefore, for any specific sample $(x,y)$:
$$ \ell(z_L(x), y) - \ell(z_S(x), y) \leq |\ell(z_L(x), y) - \ell(z_S(x), y)| \leq \|z_L(x) - z_S(x)\|_1 \,.$$
Taking the expectation over the data distribution $\mathcal{D}^T$:
$$ \mathbb{E}_{(x,y) \sim \mathcal{D}^T}[\ell(z_L(x), y) - \ell(z_S(x), y)] \leq \mathbb{E}_{(x,y) \sim \mathcal{D}^T}[\|z_L(x) - z_S(x)\|_1] \,.$$
From Assumption~\ref{app:assum_faithful_distill} ($\epsilon$-Faithful Priming), we have the condition that $\mathbb{E}_{(x,y) \sim \mathcal{D}^T}[\|z_S(x) - z_L(x)\|_1] \leq \epsilon$.
Substituting this into the inequality:
$$ \mathcal{R}_L(\mathcal{D}^T) - \mathcal{R}_S(\mathcal{D}^T) \leq \epsilon \,.$$
\end{proof}

\noindent \begin{theorem} \label{app:thm_bound_service_model_perf}
(Bounding Service Model Performance).
Combining Lemma \ref{app:lemma_perf_diff_bound} with Assumption \ref{app:assum_service_model_superiority} (Service Model Superiority), we can establish bounds on the performance of the service model $\mathcal{F}_S$ relative to the primed local model $\mathcal{F}_L$:
\begin{enumerate}[leftmargin=3em]
    \item Before visual reprogramming:
    \begin{equation}
        \mathcal{R}_L(\mathcal{D}^T) - \epsilon \leq \mathcal{R}_S(\mathcal{D}^T) \leq \mathcal{R}_L(\mathcal{D}^T)\,;
    \end{equation} 
    \item After optimal visual reprogramming:
    \begin{equation}
        \mathcal{R}_L(\mathcal{D}^T, \mathbf{P}^*) - \epsilon \leq \mathcal{R}_S(\mathcal{D}^T, \mathbf{Q}^*) \leq \mathcal{R}_L(\mathcal{D}^T, \mathbf{P}^*)\,.
    \end{equation}
    
\end{enumerate}
\end{theorem}

\begin{proof}
By Lemma~\ref{app:lemma_perf_diff_bound}, we have $\mathcal{R}_L(\mathcal{D}^T) - \mathcal{R}_S(\mathcal{D}^T) \leq \epsilon$, which rearranges to $\mathcal{R}_L(\mathcal{D}^T) - \epsilon \leq \mathcal{R}_S(\mathcal{D}^T)$.
From Assumption \ref{app:assum_service_model_superiority}, we have $\mathcal{R}_S(\mathcal{D}^T) \leq \mathcal{R}_L(\mathcal{D}^T)$.
Combining these two inequalities gives:
    $$ \mathcal{R}_L(\mathcal{D}^T) - \epsilon \leq \mathcal{R}_S(\mathcal{D}^T) \leq \mathcal{R}_L(\mathcal{D}^T) \,.$$

The reprogrammed case can be shown similarly.
\end{proof}

The above theoretical analysis provides solid justification for AReS's effectiveness by showing that faithful knowledge priming ($\epsilon$-faithful) directly translates to bounded performance differences between service and local models. Theorem \ref{app:thm_bound_service_model_perf}'s key insight reveals that traditional closed-box methods expend significant computational resources attempting to optimize the service model's performance directly, while AReS transforms this challenge by first achieving faithful priming (small $\epsilon$) and then efficiently optimizing the local model using stable first-order methods. This theoretical framework validates our empirical findings and explains why AReS can achieve competitive performance with dramatically reduced API calls and computational costs.

\section{Additional Experimental Results}
\label{app:exp_result}
\subsection{VM as a Service Model in Few-Shot Setting} \label{app:vm_fewshot}
Table~\ref{tab:vitb16_vitb32_res_fewshot} details the performance of VM adaptation using ViT-B/16 as the service model and ViT-B/32 as the local model within a challenging 16-shot per class setting. This contrasts with the full-shot experiments for VMs presented in the main paper (Tables~\ref{tab:rn101_res} and \ref{tab:vitb16_res}), which are generally preferred for VMs due to their more limited generalization compared to VLMs like CLIP. These few-shot results, however, provide a valuable benchmark for performance in data-scarce conditions. As anticipated, all evaluated methods show a decline in performance relative to the full-shot scenario. Notably, the closed-box reprogramming methods BAR and BlackVIP achieve average accuracies of only 14.6\% and 25.9\%, respectively. Despite this challenging low-data regime, our AReS method demonstrates superior performance, achieving an average accuracy of 32.8\%. This represents a significant average improvement of +18.2 percentage points over BAR and +6.9 percentage points over BlackVIP. This consistent outperformance, even with extremely limited training data, underscores AReS's effectiveness in robustly transferring knowledge from closed-box service models and significantly enhancing local model reprogramming capabilities across diverse datasets and varying data availability.

\begin{table*}[t]
\caption{
Accuracy and Efficiency comparison using RN101 (Service) in Full-shot setting.}
\vspace{-2.5mm}
\centering
\scriptsize
\resizebox{\textwidth}{!}{
\begin{tabular}{l|cccccccccc|c|r|r}
\toprule
Method & Flowers & DTD & UCF & Food & GTSRB & EuroSAT & Pets & Cars & SUN & SVHN &  Avg. & \#API (M) & Time (h) \\
\midrule
VR (glass-box)  
& \textcolor{gray!50}{42.3} & \textcolor{gray!50}{44.4} & \textcolor{gray!50}{34.4} & \textcolor{gray!50}{25.4} & \textcolor{gray!50}{53.1} & \textcolor{gray!50}{85.0} & \textcolor{gray!50}{73.9} & \textcolor{gray!50}{5.0} & \textcolor{gray!50}{20.5} & \textcolor{gray!50}{75.1} & \textcolor{gray!50}{45.9} & \textcolor{gray!50}{43.4} & 
\textcolor{gray!50}{21.9} \\
\midrule
BAR  
& 16.9 & 25.3 & 30.2 & 13.9 & 22.5 & 45.7 & 28.7 & 1.5 & 10.7 & 42.1 & 23.8 & 1,724.2 & 313.5 \\
BlackVIP (RN50)
& 21.7 & 39.6 & 31.4 & 14.9 & 26.8 & 68.3 & 61.4 & 3.5 & 12.3 & 42.6 & 32.3 & 2,586.2 & 320.5 \\
\cellcolor{gray!30}\textbf{AReS (RN50)}   
& \cellcolor{gray!30}\textbf{41.3} & \cellcolor{gray!30}\textbf{42.3} & \cellcolor{gray!30}\textbf{37.2} & \cellcolor{gray!30}\textbf{25.3} & \cellcolor{gray!30}\textbf{48.4} & \cellcolor{gray!30}\textbf{84.3} & \cellcolor{gray!30}\textbf{73.7} & \cellcolor{gray!30}\textbf{5.1} & \cellcolor{gray!30}\textbf{19.7} & \cellcolor{gray!30}\textbf{72.9} & \cellcolor{gray!30}\textbf{45.0}
& \cellcolor{gray!30}\textbf{0.2} & \cellcolor{gray!30}\textbf{8.7} \\
BlackVIP (ViT-B/32)
& 26.7 & 37.7 & 30.2 & 13.3 & 33.8 & 67.3 & 57.9 & 3.0 & 14.5 & 34.5 & 31.9 & 2,586.2 & 418.0 \\
\cellcolor{gray!30}\textbf{AReS  (ViT-B/32)}     
& \cellcolor{gray!30}\textbf{51.3} & \cellcolor{gray!30}\textbf{45.8} & \cellcolor{gray!30}\textbf{41.5} & \cellcolor{gray!30}\textbf{30.2} & \cellcolor{gray!30}\textbf{57.9} & \cellcolor{gray!30}\textbf{94.9} & \cellcolor{gray!30}\textbf{63.3} & \cellcolor{gray!30}\textbf{5.1} & \cellcolor{gray!30}\textbf{25.6} & \cellcolor{gray!30}\textbf{83.6} & \cellcolor{gray!30}\textbf{49.9} & \cellcolor{gray!30}\textbf{0.2} & 
\cellcolor{gray!30}\textbf{22.3} \\
\bottomrule
\end{tabular}
}
\label{tab:rn101_res}
\vspace{-2mm}
\end{table*}

\begin{table*}[t]
\caption{Accuracy comparison using ViT-B/16 (Service) and ViT-B/32 (Local) in 16-shot setting.}
\centering
\scriptsize
\resizebox{\textwidth}{!}{
\begin{tabular}{l|cccccccccc|c}
\toprule
Method & Flowers & DTD & UCF & Food & GTSRB & EuroSAT & Pets & Cars & SUN & SVHN &  Avg. \\
\midrule
VR (glass-box)  
& \textcolor{gray!50}{55.0} & \textcolor{gray!50}{44.7} & \textcolor{gray!50}{42.0} & \textcolor{gray!50}{18.0} & \textcolor{gray!50}{17.4} & \textcolor{gray!50}{70.6} & \textcolor{gray!50}{70.7} & \textcolor{gray!50}{5.7} & \textcolor{gray!50}{29.6} & \textcolor{gray!50}{28.4} & \textcolor{gray!50}{38.2} \\
\midrule
BAR  
& 9.7 & 16.3 & 23.5 & 8.1 & 4.8 & 34.2 & 21.5 & 1.0 & 7.6 & 19.6 & 14.6 \\
BlackVIP 
& 16.1 & \textbf{36.1} & 33.8 & 12.5 & 9.1 & 54.9 & 54.5 & 3.9 & 22.1 & 16.2 & 25.9 \\
\rowcolor{gray!30}
\textbf{AReS (Ours)} 
& \textbf{46.0} & 35.5 & \textbf{34.3} & \textbf{12.7} & \textbf{19.6} & \textbf{71.1} & \textbf{54.3} & \textbf{4.2} & \textbf{23.0} & \textbf{27.3} & \textbf{32.8} \\
\bottomrule
\end{tabular}
}
\label{tab:vitb16_vitb32_res_fewshot}
\end{table*}

\begin{table*}[t]
\caption{Accuracy comparison between Local VR (ViT-B/16) and our method on VLMs using CLIP ViT-B/16 (Service) and ViT-B/16 (Local) in a 16-shot setting.}
\centering
\scriptsize
\resizebox{\textwidth}{!}{
\begin{tabular}{l|cccccccccc|c}
\toprule
Method & Flowers & DTD & UCF & Food & GTSRB & EuroSAT & Pets & Cars & SUN & SVHN &  Avg. \\
\midrule
Zero-shot 
& 71.3 & 43.9 & 66.9 & \textbf{85.9} & 21.0 & 47.9 & \textbf{89.1} & \textbf{65.2} & 62.6 & 17.9 & 57.2 \\
Local VR  
& 55.0 & 44.7 & 42.0 & 18.0 & 17.4 & 70.6 & 70.7 & 5.7 & 29.6 & 28.4 & 38.2  \\
\cellcolor{gray!30}\textbf{AReS (Ours)}    
& \cellcolor{gray!30}\textbf{86.6} & \cellcolor{gray!30}\textbf{48.2} & \cellcolor{gray!30}\textbf{67.1} & \cellcolor{gray!30}68.8 & \cellcolor{gray!30}\textbf{39.4} & \cellcolor{gray!30}\textbf{85.7} & \cellcolor{gray!30}88.9 & \cellcolor{gray!30}43.2 & \cellcolor{gray!30}\textbf{62.8} & \cellcolor{gray!30}\textbf{63.2} & \cellcolor{gray!30}\textbf{65.4} \\
\bottomrule
\end{tabular}
}
\label{tab:clip_local_vr}
\end{table*}

\subsection{AReS's Impact on Local Model Performance}
\label{app:local_wb_vr}

To quantify AReS's enhancement of local model capabilities, we compare it against directly performing glass-box VR on the local model (termed "Local VR"). For AReS, we assume a pre-trained local encoder, consistent with methods like BlackVIP. For this evaluation of the local encoder only, we additionally assume a pre-trained linear layer on the source domain. As shown in Table~\ref{tab:clip_local_vr} for VLMs (CLIP ViT-B/16 service, ViT-B/16 local, 16-shot), AReS achieves a 65.4\% average accuracy, a substantial $+27.2\%$ improvement over Local VR's 38.2\%. This highlights that AReS's priming significantly boosts the local ViT-B/16's reprogrammability beyond its standalone capacity.

\begin{table}
\centering
\caption{Accuracy comparison between Local VR and our method on VMs in Full-shot setting.}
\label{tab:vm_local_vr_fullshot}
\scriptsize
\begin{tabular}{ll|cc}
\toprule

\textbf{Service} & \textbf{Local}& \textbf{Local VR}& \textbf{Ours}\\
\midrule
\multirow{2}{*}{ViT-B/16} & ViT-B/32 & 45.3 & \cellcolor{gray!30}{\textbf{50.4}} \\
& RN50     & 43.9 & \cellcolor{gray!30}{\textbf{45.9}} \\
\midrule
\multirow{2}{*}{RN101}    & ViT-B/32 & 45.3 & \cellcolor{gray!30}{\textbf{49.9}} \\
& RN50     & 43.9 & \cellcolor{gray!30}{\textbf{45.0}} \\
\bottomrule
\end{tabular}
\end{table}

A similar advantage for AReS is observed with VMs in the full-shot setting, as detailed in Table~\ref{tab:vm_local_vr_fullshot}. For instance, with a ViT-B/16 service model, AReS improves upon Local VR by $+5.1\%$ (50.4\% vs. 45.3\%) when using a ViT-B/32 local model, and by $+2.0\%$ (45.9\% vs. 43.9\%) with an RN50 local model. Consistent gains are also seen with an RN101 service model. These results across both VLM and VM configurations underscore the critical role of AReS's initial priming phase. By effectively transferring knowledge from the more powerful service model, AReS significantly elevates the local model's baseline performance and its potential for reprogramming, demonstrating a more effective utilization of combined model strengths.

\begin{table}[t]
\centering
\caption{Comparison of trainable parameters and accuracy on the EuroSAT dataset. More parameters do not correlate with better performance for ZOO-based methods.}
\label{tab:params}
\scriptsize
\begin{tabular}{l|rc}
\toprule
\textbf{Method} & \textbf{\# Trainable Params} & \textbf{Accuracy (\%)} \\
\midrule
VP w/ SPSA-GC & 69K & 70.9 \\
BAR & 37K & 77.3 \\
BlackVIP & 9K & 73.3 \\
\rowcolor{gray!30}
\textbf{AReS (Ours)} & \textbf{21K} & \textbf{85.7} \\
\bottomrule
\end{tabular}
\end{table}

\begin{table}[t]
\centering
\caption{Comparison with an enhanced BlackVIP baseline on the EuroSAT dataset. The results confirm that AReS's performance gain is primarily from its superior two-stage framework.}
\label{tab:enhanced_blackvip}
\scriptsize
\begin{tabular}{l|c}
\toprule
\textbf{Method} & \textbf{Accuracy (\%)} \\
\midrule
Local VR & 70.6 \\
BlackVIP & 73.3 \\
BlackVIP w/ Local Tuning & 74.1 \\
\rowcolor{gray!30}
\textbf{AReS (Ours)} & \textbf{85.7} \\
\bottomrule
\end{tabular}
\end{table}

\subsection{Dissecting the Source of Performance Gain}
The superior performance of AReS stems not from a higher parameter count or the mere inclusion of local tuning, but from its novel two-stage framework. We demonstrate that the initial priming stage is an indispensable component and that the overall design is highly parameter-efficient.

First, an analysis of parameter efficiency reveals that for ZOO-based methods, a larger number of trainable parameters does not guarantee better performance and can even be detrimental. As shown in Table~\ref{tab:params}, a baseline using SPSA-GC with a large visual prompt (69K parameters) achieves a lower accuracy than more compact models. This is likely because noisy gradient estimates are less stable in higher-dimensional spaces. In contrast, AReS's stable, glass-box optimization allows for the effective tuning of a compact prompt, confirming its effectiveness is due to a superior optimization strategy, not parameter quantity.

Second, the performance gain is critically dependent on the initial priming stage. Our component analysis on EuroSAT (Table~\ref{tab:ablation_component} in the main paper) shows that local visual reprogramming alone (Local VR) achieves 70.6\% accuracy. The full AReS method, which includes priming, elevates this to 85.7\%, a substantial +15.1\% improvement demonstrating that the knowledge transferred during priming is essential for unlocking the local model's full potential.To further validate this, we created an enhanced version of BlackVIP that incorporates a local pre-training stage. To mimic a local pre-training stage, we trained its prompt-generating Coordinator decoder with a reconstruction loss, as direct supervision is not possible without a ground-truth prompt. As shown in Table~\ref{tab:enhanced_blackvip}, this local optimization provides only a marginal improvement to BlackVIP (+0.8\%). AReS still significantly outperforms this enhanced baseline by +11.6\%, confirming that the performance gain originates from our superior two-stage framework, where priming makes the local model significantly more amenable to reprogramming.

\subsection{Performance Scaling with Data Availability}
A key question is whether AReS's advantage over a strong, locally-trained baseline persists as more training data becomes available. To investigate this, we conducted an ablation study on the EuroSAT dataset, varying the number of training samples from 4 to 32 shots. The results in Table~\ref{tab:scaling} show that AReS consistently outperforms both the ZOO-based baseline (BlackVIP) and a strong local baseline (Local VR+LP) across all data settings.

While the local baseline's performance improves steadily with more data, it never surpasses AReS. Our method is not only competitive in low-data regimes but also scales more effectively as data increases. For instance, at 32 shots, AReS achieves 91.6\% accuracy, maintaining a significant +5.1\% lead over the 86.5\% from the local baseline. This confirms that the initial priming stage provides a durable advantage that local training alone cannot replicate, equipping the local model with a superior foundation that enables a higher performance ceiling.

\begin{table}[t]
\centering
\caption{Accuracy (\%) comparison on the EuroSAT dataset across different few-shot settings. AReS consistently outperforms the strong local baseline (Local VR+LP) as data availability increases.}
\label{tab:scaling}
\scriptsize
\begin{tabular}{l|cccc}
\toprule
\textbf{Method} & \textbf{4-shot} & \textbf{8-shot} & \textbf{16-shot} & \textbf{32-shot} \\
\midrule
BlackVIP & 69.3 & 71.7 & 73.3 & 72.9 \\
Local VR+LP & 59.7 & 68.1 & 80.1 & 86.5 \\
\rowcolor{gray!30}
\textbf{AReS (Ours)} & \textbf{66.1} & \textbf{73.4} & \textbf{85.7} & \textbf{91.6} \\
\bottomrule
\end{tabular}
\end{table}

\begin{table}[t]
\centering
\small
\caption{Illustrative Breakdown of AReS's Performance Composition on Flowers102. The local model for both scenarios is ViT-B/32 (glass-box Acc: 38.7\%). ``Kept correct" refers to samples the local model classified correctly that AReS also classifies correctly. ``Newly learned" refers to samples the local model misclassified that AReS now classifies correctly.}
\label{tab:synergy}
\scriptsize
\resizebox{\textwidth}{!}{%
\begin{tabular}{@{}l|rr|rr@{}}
\toprule
\multirow{2}{*}{\textbf{Initial State}}& \multicolumn{2}{c|}{\textbf{Service: RN101 (42.3\%)}} & \multicolumn{2}{c}{\textbf{Service: ViT-B/16 (69.6\%)}} \\ 
\cmidrule(lr){2-3} \cmidrule(lr){4-5}
& \textbf{\% of Dataset} & \textbf{How AReS Performed} & \textbf{\% of Dataset} & \textbf{How AReS Performed} \\
\midrule
Both Service \& Local Correct & 30.0\% & 29.5\% kept correct & 35.0\% & 34.0\% kept correct \\
Only Local Correct & 8.7\% & 7.0\% kept correct & 3.7\% & 0.5\% kept correct \\
Only Service Correct & 12.3\% & 10.0\% newly learned & 34.6\% & 19.0\% newly learned \\
Both Incorrect & 49.0\% & 4.8\% newly learned & 26.7\% & 0.1\% newly learned \\
\midrule
\rowcolor{gray!30}
\textbf{AReS Final Accuracy} & \multicolumn{2}{c|}{\textbf{51.3\%}} & \multicolumn{2}{c}{\textbf{53.6\%}} \\
\bottomrule
\end{tabular}
}
\end{table}

\subsection{Synergistic Effect of Combining Model Knowledge}
\label{sec:app_synergy}
In certain configurations, AReS's performance can surpass that of a glass-box reprogramming approach on the service model itself. This outcome stems from a synergistic effect, where our framework effectively combines the distinct knowledge of the powerful service model with the unique inductive biases of the local model. 
This process of navigating differing predictive behaviors to harness more reliable learning outcomes shares conceptual similarities with recent advances in trustworthy multi-view classification~\cite{pmlr-v267-lu25a}.
Unlike methods~\citep{balckvip_CVPR23} that use the local model as a simple feature generator, AReS’s two-stage design first primes and then leverages the local model's capabilities during reprogramming. This process allows the final model to correctly classify samples that neither the service model nor the standalone local model could handle individually.

This synergistic effect is demonstrated quantitatively in our analysis of the Flowers102 dataset (Table~\ref{tab:synergy}), where we compare scenarios with service models of varying strength. When the service model (RN101) has capabilities comparable to the local model, AReS’s final accuracy surpasses both, largely by learning to classify an additional 4.8\% of samples that were incorrect for both models initially. Conversely, when the service model (ViT-B/16) is significantly stronger, AReS's role shifts to highly effective knowledge transfer, yielding a remarkable +14.9\% absolute improvement over reprogramming the local model alone. This analysis confirms that AReS facilitates a potent combination of knowledge, creating synergistic outcomes not possible with other approaches.

\subsection{Performance on Large-Scale Benchmarks: ImageNet}
\label{app:imagenet_results}

To further validate our method's effectiveness on a large-scale, stable benchmark, we conducted experiments on ImageNet~\cite{imagenet} in a 16-shot setting, using CLIP ViT-B/16 as the service model. The results, presented in Table~\ref{tab:imagenet_results}, demonstrate AReS's significant superiority. While prior ZOO-based methods like BAR and VP w/ SPSA-GC degrade performance compared to the zero-shot baseline, and the SOTA BlackVIP provides only a marginal +0.4\% gain, AReS achieves a remarkable 80.1\% accuracy. This represents a substantial +13.4\% improvement over the zero-shot baseline and a +13.0\% gain over the next best competitor, BlackVIP. Notably, AReS even surpasses the glass-box reprogramming baseline, further highlighting the powerful synergistic effect of combining the service model's knowledge with the local model's inductive biases. This is particularly notable as it demonstrates that when a strong local model is available, AReS can effectively employ this advantage to achieve performance superior to both the service model and other local-model-based methods. In contrast, the BlackVIP baseline, despite utilizing the same local encoder architecture, fails to leverage its capabilities, providing only a negligible improvement. These compelling results on a challenging, large-scale dataset confirm that our conceptual shift away from the ZOO-based paradigm is a more robust and effective strategy for closed-box model adaptation.

\begin{table}[t]
\centering
\caption{Accuracy (\%) comparison on ImageNet in a 16-shot setting with CLIP ViT-B/16 as the service model. AReS significantly outperforms all baselines.}
\label{tab:imagenet_results}
\scriptsize
\begin{tabular}{l|cc}
\toprule
\textbf{Method} & \textbf{ImageNet Acc (\%)} & \textbf{Gain over Zero-shot (\%)} \\
\midrule
VR (glass-box) & \textcolor{gray!50}{67.4} & \textcolor{gray!50}{+0.7} \\
Zero-shot & 66.7 & - \\
\midrule
BAR & 64.6 & -2.1 \\
VP w/ SPSA-GC & 62.3 & -4.4 \\
BlackVIP & 67.1 & +0.4 \\
\rowcolor{gray!30}
\textbf{AReS (Ours)} & \textbf{80.1} & \textbf{+13.4} \\
\bottomrule
\end{tabular}
\end{table}

\subsection{Performance Analysis on Challenging VLM Scenarios}
\label{app:failure_analysis}
While AReS demonstrates considerable average performance improvements with unparalleled efficiency, its efficacy with VLMs like CLIP as a service can exhibit variability on particularly challenging datasets, notably Food101 and Cars, as indicated in Table~\ref{tab:clip_res}. The inherent difficulty of these datasets for reprogramming is underscored by the performance of even highly capable baselines. For instance, glass-box VR performed directly on the CLIP ViT-B/16 service model achieves 81.6\% on Food101, which is below the zero-shot performance of 85.9\%. On Cars, glass-box VR obtains 66.2\%, only a marginal improvement over the 65.2\% zero-shot accuracy. This suggests a ceiling for reprogramming efficacy on these complex domains, even with full model access. Consequently, ZOO-based closed-box methods like BAR and BlackVIP also struggle; on Food101, BAR (84.4\%) and BlackVIP (85.9\%) offer negligible to no improvement over zero-shot performance. Similarly, on Cars, BAR (63.0\%) underperforms the zero-shot baseline, and BlackVIP (65.4\%) provides only a minimal gain. These results highlight that substantial API call volumes and computational efforts by existing closed-box methods do not reliably translate into significant performance gains on such challenging datasets.

In this context, AReS's performance (68.8\% on Food101 and 43.2\% on Cars in Table~\ref{tab:clip_res}) is logically influenced by the capabilities of the local ViT-B/16 model. Given that our approach entirely avoids API calls during inference and subsequent reprogramming, its success is intrinsically tied to how well the primed knowledge empowers the local model for the specific downstream task. On these particularly demanding datasets, the primed local model may not fully capture the intricate features that the larger service model leverages for its strong zero-shot performance, leading to a performance gap. However, it is crucial to evaluate AReS's contribution in terms of enhancing the local model's standalone reprogrammability. As shown in Table~\ref{tab:clip_local_vr}, direct Local VR on the ViT-B/16 achieves only 18.0\% on Food101 and a mere 5.7\% on Cars. AReS elevates these figures to 68.8\% and 43.2\% respectively, marking substantial improvements of $+50.8\%$ and $+37.5\%$. This significant enhancement of the local model's utility, achieved with minimal, one-time API interaction and drastically reduced computation, is of immense practical value, particularly in scenarios where continuous and costly reliance on service APIs is untenable. Addressing the remaining performance gap on these specific challenging datasets, possibly through advancements in primed techniques or by employing more capable local architectures, presents an interesting direction for future research.

\begin{figure}[t]
    \centering
    \includegraphics[width=1.0\linewidth]{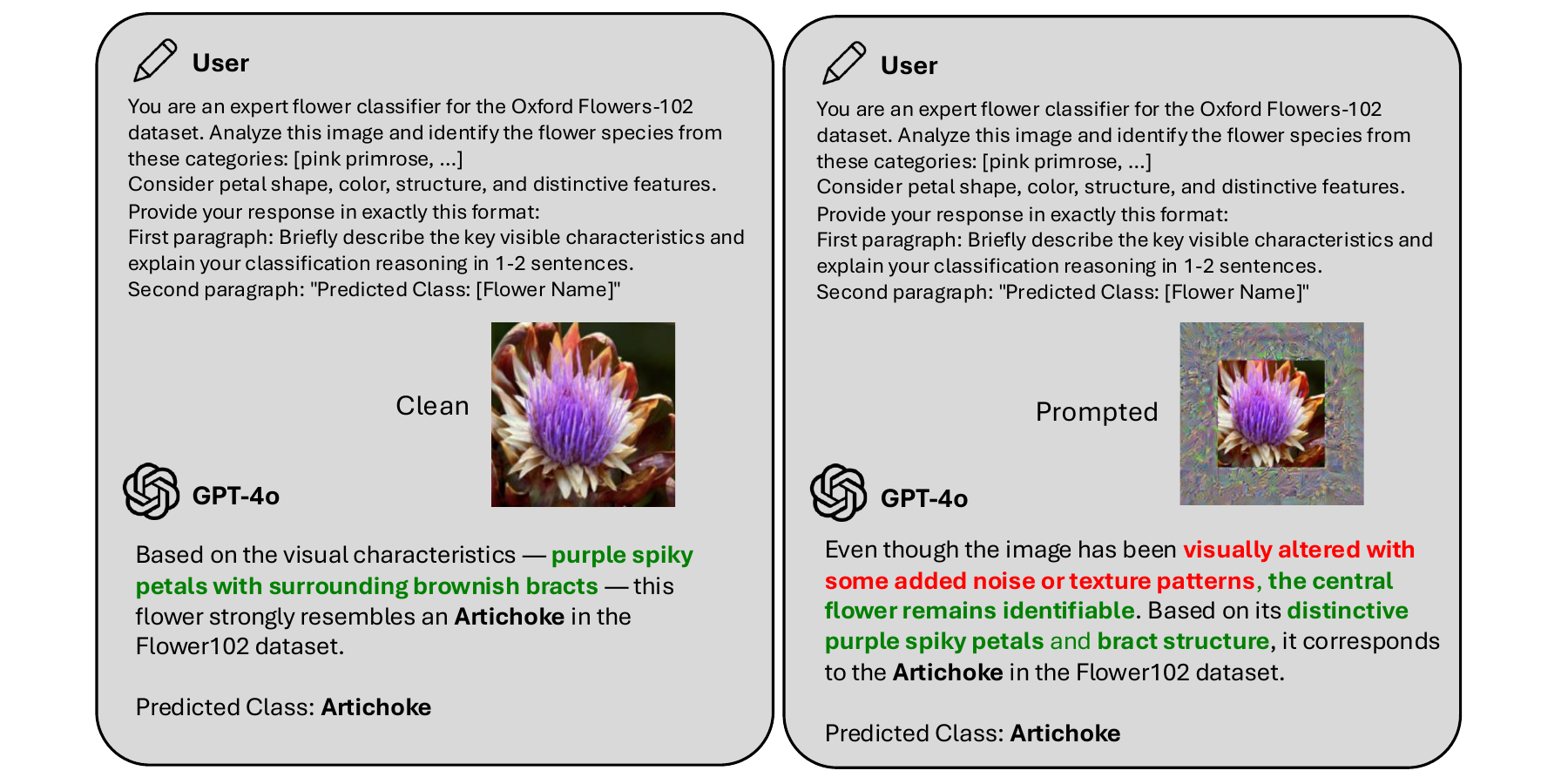}
    \caption{An example of GPT-4o's robustness to the input perturbations used in reprogramming. When presented with a clean image (left) and a prompted image (right), GPT-4o's textual reasoning explicitly acknowledges the visual alteration but correctly identifies the flower's key features in both cases. This demonstrates that the model is often invariant to the input noise that ZOO-based methods rely on, motivating a strategic shift in adaptation methods.}
    
    \label{fig:AReS_gpt_example}
\end{figure}

\subsection{Results on MLLMs and Real-world APIs}
\label{app:mllm_gpt}
To validate AReS’s practical effectiveness on modern models, we conducted a targeted evaluation using the EuroSAT 16-shot benchmark on three distinct services: the open-source MLLM LLaVA, the proprietary VLM GPT-4o, and the commercial VM API \footnote{\texttt{https://www.\,clarifai.com/}}{Clarifai}. Since MLLMs do not produce output probabilities by default, we instructed the models to act as a classifier and return a confidence score for each class (e.g., \texttt{"For the classes [...], provide a confidence score from 0 to 1 for each."}). Our evaluation reveals a critical limitation of traditional BMR on these services. As illustrated in Figure~\ref{fig:AReS_gpt_example}, powerful models like GPT-4o are highly robust to the input perturbations central to ZOO-based reprogramming. The model's textual reasoning shows it can ``see past'' the visual noise to the underlying image content, rendering the ZOO process that relies on these perturbations ineffective. This finding, supported by quantitative results where ZOO fails to improve over the zero-shot baseline (Table~\ref{tab:mllm_api}), requires a strategic shift away from perturbation-based adaptation for modern APIs.

In contrast, AReS's two-stage framework is architecturally immune to this issue. By performing a one-time priming step and shifting all subsequent reprogramming to a local model, AReS's success is not dependent on the service model's sensitivity to input noise. This allows it to achieve a remarkable \textbf{+27.8\%} gain over the zero-shot baseline on GPT-4o, a task where other methods stagnate. Furthermore, on the commercial Clarifai API---a closed-vocabulary setting where zero-shot is not an option---AReS again achieves the highest accuracy (83.2\%) at a fraction of the cost (\$0.20), making it over \textbf{300x cheaper} than BlackVIP (\$67.30). These results confirm that AReS's strategy provides a more robust, effective, and economically viable solution for adapting modern closed-box models.


\subsection{Computational Efficiency Analysis}
\label{app:computational_efficiency}
To evaluate AReS's local resource trade-offs, we report peak GPU memory and per-image inference latency measured on an NVIDIA RTX 3090 for the EuroSAT benchmark (Table~\ref{tab:computational_efficiency}). AReS requires comparable GPU memory to existing closed-box baselines such as BlackVIP while achieving +12.4\% higher accuracy. Notably, AReS achieves the fastest inference latency (60ms) by running entirely locally after the one-time priming stage, whereas API-based methods incur additional network latency on top of their local computation. This modest local computational cost represents a reasonable trade-off for eliminating API dependency and enabling cost-free, low-latency inference at deployment time.

\begin{table}[h]
\centering
\caption{GPU memory and inference latency comparison on EuroSAT with NVIDIA RTX 3090. AReS achieves the highest accuracy and fastest inference by running entirely locally after priming.}
\label{tab:computational_efficiency}
\small
\begin{tabular}{@{}lrrc@{}}
\toprule
Method & GPU Mem. (MB) & Latency (ms) & Acc. (\%) \\
\midrule
\textcolor{gray}{VR (glass-box)} & \textcolor{gray}{11,937} & \textcolor{gray}{-} & \textcolor{gray}{90.9} \\
Zero-shot & 0 & 120 & 47.9 \\
BAR & 1,649 & 122 & 77.3 \\
BlackVIP & 2,428 & 170 & 73.3 \\
\textbf{AReS (Ours)} & \textbf{2,781} & \textbf{60} & \textbf{85.7} \\
\bottomrule
\end{tabular}
\end{table}

\subsection{Robustness to Domain Gap Between Service and Local Models}
\label{app:domain_gap}
To evaluate AReS's robustness when there is a substantial domain gap between the service model and the local model, we conduct an experiment using a domain-specific dermatology service model (DermCLIP, pretrained on the Derm1M dataset~\citep{yan2025derm1m}) paired with a generic local encoder (ViT-B/16 pretrained on natural images) on the HAM10000 skin lesion dataset~\citep{tschandl2018ham10000} under the 16-shot setting. This setup represents a challenging scenario where the service model possesses highly specialized domain knowledge (medical dermatology), while the local model has only been exposed to generic natural image features.
Recent works have highlighted the importance of domain-specific datasets in medical imaging~\citep{wang2025doctor, 10.1145/3746027.3758241}, yet deploying such specialized models often requires costly API access. AReS addresses this by transferring specialized knowledge to a generic local model through one-time priming.

Despite the substantial domain mismatch (medical vs.\ natural images), AReS outperforms both the zero-shot baseline and BlackVIP (Table~\ref{tab:domain_gap}), confirming that it is not intrinsically tied to domain-matched encoders. The priming mechanism successfully aligns the generic local feature space to the specialized service model even without domain-specific pretraining. The smaller gains compared to domain-aligned settings are consistent with our theoretical analysis: extreme domain gaps increase the priming bound $\varepsilon$ (Assumption~\ref{app:assum_faithful_distill}), and performance is ultimately constrained by the local encoder's capacity to represent specialized features (e.g., skin lesion patterns).

\begin{table}[h]
\centering
\caption{Domain gap experiment: DermCLIP (pretrained on Derm1M) $\rightarrow$ ViT-B/16 (pretrained on ImageNet) on HAM10000 under the 16-shot setting.}
\label{tab:domain_gap}
\small
\begin{tabular}{@{}lccc@{}}
\toprule
Method & Zero-shot & BlackVIP & \textbf{AReS (Ours)} \\
\midrule
Acc. (\%) & 62.5 & 63.2 & \textbf{63.6} \\
\bottomrule
\end{tabular}
\end{table}

\subsection{Robustness to Partial API Outputs}
\label{app:partial_api}
Many real-world APIs return only partial outputs, such as top-$k$ predictions rather than the full probability distribution over all classes. To evaluate AReS's robustness under such constraints, we conduct an ablation study restricting the service model's output to: (i) \emph{top-$k$ soft labels} (top-$k$ class names with their confidence scores) and (ii) \emph{top-$k$ hard labels} (only an ordered top-$k$ list without confidence scores). For both settings, we convert the partial output into a usable training target by keeping the revealed top-$k$ information and distributing the remaining probability mass uniformly across the unrevealed classes. For hard labels, we assign a total mass of $c(k)=\frac{k}{k+1}$ to the top-$k$ set and split it across the $k$ classes by rank using a decay weight $w_r \propto \frac{1}{r}$, assigning higher mass to higher-ranked classes.

Results on EuroSAT (10 classes) are shown in Table~\ref{tab:partial_api}. AReS demonstrates graceful degradation as $k$ decreases, while remaining close to full-output performance for moderate $k$ values (e.g., $k=5$ achieves 85.0\% vs.\ 85.7\% with full output). This constraint equally affects all methods relying on output probabilities (including BAR and BlackVIP), as their optimization objectives similarly depend on full prediction distributions.

\begin{table}[h]
\centering
\caption{Robustness to partial API outputs (top-$k$) on EuroSAT. ``Soft'' uses top-$k$ class names with confidence scores; ``Hard'' uses only the ordered top-$k$ list.}
\label{tab:partial_api}
\small
\begin{tabular}{@{}lcccc@{}}
\toprule
Top-$k$ & 1 & 2 & 5 & all (10) \\
\midrule
Soft & 80.1 & 82.3 & 85.0 & 85.7 \\
Hard & 70.7 & 74.6 & 78.9 & 80.6 \\
\bottomrule
\end{tabular}
\end{table}

\section{Further Discussion on Contributions and Design Rationale} \label{app:further_discussion}
We further clarifies the core contributions of AReS, its design rationale, and its positioning relative to alternative paradigms.

\subsection{On Novelty: Priming for Reprogrammability vs. Traditional Distillation}
A key contribution of AReS is the high-level `prime-then-reprogram' conceptual framework. This framework is flexible, and the priming stage is not constrained to a single implementation.
While this work successfully uses an objective function inspired by knowledge distillation (KL divergence), the core idea is to enhance the local model's reprogrammability. Any future technique that can effectively prime a local model's amenability to reprogramming from a closed-box source could be integrated into this framework.
This focus on ``\textbf{priming for reprogrammability}" is what fundamentally distinguishes AReS from traditional distillation (also in Sec.~\ref {sec:remark_kd}):

\begin{enumerate}[leftmargin=3em]
    \item Traditional KD aims to create a final, high-performance student model. This approach often fails in BMR scenarios, particularly when adapting standard Vision Models (VMs). For instance, if the service model is a VM (e.g., pre-trained on ImageNet) and the target is Flowers102, their \textbf{label spaces are disjoint}. One cannot directly distill a Flowers102 classifier from an ImageNet classifier.
    \item AReS's Priming is \textbf{solely a preparatory step}. It is not intended to solve the final task. Instead, it operates in the service model's label space (e.g., ImageNet) to make the local model API-aware when given downstream data (e.g., Flowers102 images).
\end{enumerate}
This novel preparatory focus is what enables the second, local reprogramming stage to be highly effective, even in challenging disjoint-label VM scenarios where \emph{traditional distillation is not applicable}.

We note that knowledge distillation has been explored extensively in both vision and language domains, including step-by-step distillation from large language models~\citep{hsieh-etal-2023-distilling, ho2023large}, multi-chain-of-thought consistent distillation~\citep{chen-etal-2023-mcc}, and black-box few-shot knowledge distillation~\citep{nguyen2022black}. While these methods have shown success in their respective settings, they fundamentally aim to produce a student model that solves the same task as the teacher. In contrast, AReS's priming stage operates across disjoint label spaces and serves only as a preparatory mechanism for subsequent reprogramming, representing a fundamentally different use of the distillation objective.

\subsection{On Practicality: Cost-Free Inference as a Core Design Goal}
A core design choice of AReS is to perform a single-pass knowledge transfer, which then \textbf{intentionally eliminates all subsequent API dependency} during inference. This is not a shortcoming but the central advantage of our framework, enabling practical deployment in on-device, edge, offline, or cost-sensitive scenarios.
The practical motivation for this new paradigm is threefold:

\begin{enumerate}[leftmargin=3em] \item \textbf{ZOO-based Methods are Less Effective for Modern APIs:} We find that the ZOO-based paradigm is becoming ineffective on modern, robust APIs. Powerful models like GPT-4o are often less sensitive to the input perturbations that ZOO methods rely on . Our experiments (Table~\ref {tab:mllm_api}) confirm this: ZOO-based methods provide little to no improvement over the zero-shot baseline on GPT-4o. In contrast, AReS achieves a \textbf{+27.8\% gain}, succeeding precisely where the previous paradigm fails. This provides a strong practical reason to shift to a local model primed by the service API.

\item \textbf{Synergistic Performance:} The AReS framework can unlock synergistic effects by combining the knowledge of the powerful service model with the inductive biases of the local model. As shown in our analysis (Table~\ref{tab:synergy}), this allows AReS to correctly classify samples that neither the service model nor the standalone local model could handle individually. This synergy is so effective that AReS can even \textbf{outperform the glass-box VR performance of the service model}—the theoretical performance ceiling for any ZOO-based method.

\item \textbf{Enabling Real-World Deployment:} This design unlocks a wide range of practical scenarios where perpetual API access is not feasible or desirable, such as on-device or edge-computing applications, scenarios requiring real-time or offline adaptation, and cost-sensitive applications where a $>$99.99\% reduction in API calls is a primary requirement.
\end{enumerate}

Our real-world API experiments (Table~\ref{tab:mllm_api}) confirm this total practical value. On the commercial Clarifai API, AReS achieves superior accuracy (83.2\%) for just \$0.20, while BlackVIP costs \$67.30 for lower accuracy (72.1\%).

\subsection{On Efficacy: AReS vs. Standalone Local Model Reprogramming}
To isolate the contribution of our priming stage, we conducted extensive component analyses comparing AReS to a Local VR baseline. This baseline represents the naive approach of simply reprogramming the local model without any priming.

The results conclusively demonstrate that our performance gain is not merely from ``just training a small network," but from the synergistic ``\textbf{prime-then-reprogram}" framework.

\begin{enumerate}[leftmargin=3em]
    \item On EuroSAT (Table~\ref{tab:ablation_component}): Local VR (baseline) achieves 70.6\% accuracy. The full AReS framework achieves 85.7\%—a +15.1\% gain directly attributable to the priming stage.
    \item On VLMs (Table~\ref{tab:vitb16_vitb32_res_fewshot}): AReS achieves a 65.4\% average, a +27.2\% improvement over the 38.2\% from Local VR.
    \item On VMs (Table~\ref{tab:vm_local_vr_fullshot}): AReS (50.4\%) improves upon Local VR (45.3\%) by +5.1\%.
\end{enumerate}
This empirically proves that our priming stage successfully transfers knowledge from the service model, making the local model significantly more amenable to effective reprogramming.

\subsection{On the Theoretical Framework}
Our theoretical analysis provides the formal justification for why our two-stage approach is effective.
The assumption of $\epsilon$-Faithful Priming (Assumption~\ref{app:assum_faithful_distill}) is not a given; it is the explicit goal of our Priming stage. Our framework is constructive:

\begin{enumerate}[leftmargin=3em]
    \item Principle: We posit that if $\epsilon$-faithful priming can be achieved, the unstable, query-heavy closed-box optimization on the service model $\mathcal{F}_{S}$ can be provably bounded by an efficient, stable, first-order optimization on the local model $\mathcal{F}_{L}$ (Theorem~\ref{app:thm_bound_service_model_perf}).
    \item Mechanism: The Primeing stage is the practical mechanism we designed to achieve this faithful priming (i.e., a small $\epsilon$).
    We also empirically validate the robustness of this mechanism in our ablation study (Fig.~\ref{fig:ablation_ditil_loss}). The results show that our practical, probability-only (closed-box) method achieves the same strong performance as a variant using logits (translucent-box), confirming that our theoretical analysis is well-grounded and its assumptions do not create a gap with our practical implementation.
    \item Result: Our extensive empirical results—especially the +27.8\% gain on GPT-4o where ZOO methods fail —demonstrate that this mechanism succeeds in practice.
\end{enumerate}
The theory, therefore, explains why our practical approach of shifting optimization to a local model is a sound and effective strategy for leveraging the service model's capabilities. To the best of our knowledge, this is \textbf{the first work to provide a formal theoretical analysis for model reprogramming} in the context of closed-box service models.

\subsection{Positioning Within the Broader Visual Prompting Landscape}
\label{app:broader_prompting}
AReS is situated within a rapidly growing landscape of visual prompting and parameter-efficient fine-tuning (PEFT) methods~\citep{jia2022visual, tu2023visual, 10.1145/3746027.3754858, xiao2025visual, mai2025lessons, wang2025ctr, xiao2026not}. 
Furthermore, this push towards highly efficient adaptation and data utilization conceptually parallels recent breakthroughs in broader efficient learning paradigms, such as dataset distillation \cite{zhao2026hieramp, zhao2025taming} and efficient generative modeling \cite{zhao2026s2dit}, and LLM-driven active learning \cite{qiyuanyuan2026mollia}.
These works have advanced prompt design in diverse directions, including instance-aware prompting, intermediate-representation-based transfer, and systematic studies of efficiency-performance trade-offs in visual recognition. Beyond standard white-box settings, recent efforts have also explored visual prompting under distribution shift and restricted model access~\citep{zhang2025dpcore, zhang2025otvp, park2025zip}, extending the prompting paradigm to scenarios where gradient information is unavailable or the model must adapt continually.

While these methods primarily operate in settings where the model is either fully accessible or adapted at test time, AReS addresses a distinct challenge: efficiently transferring knowledge from a closed-box service model to a local model through one-time priming, followed by local glass-box reprogramming. This positions AReS as complementary to these approaches, offering a practical pathway for leveraging powerful but inaccessible service models.

\section{Limitations}
\label{app:limitations}
Our AReS method, while offering benefits in efficiency and local model enhancement, has certain limitations. In extreme data scarcity scenarios, such as 1/2-shot learning, the one-time priming phase may be constrained (see Fig.~\ref{fig:ablation_shot}). This can challenge the $\epsilon$-faithful primed assumption in our theoretical analysis, potentially affecting the optimality of subsequent local VR and leading to suboptimal performance. Furthermore, AReS's performance is inherently affected by the representational capacity of the chosen local model. If the local model, even after priming, cannot capture the complexities of a particularly challenging downstream task, AReS may not achieve the performance levels of methods that continuously leverage the full computational power and feature richness of the larger service model throughout adaptation (discussed in Sec.
\ref{app:failure_analysis}).
An interesting future direction is to explore connections between AReS and black-box test-time adaptation~\citep{zhang2026adapting, maharana2026continual}, where models must adapt to distribution shifts at test time without access to the model's internals. Combining the strengths of one-time priming with online adaptation strategies could further enhance the robustness and flexibility of closed-box model adaptation.

\end{document}